\documentclass{article}

\PassOptionsToPackage{numbers, compress}{natbib}
\usepackage{times,wrapfig,amsmath,amsfonts,bm,color,enumitem,algorithm,algpseudocode}

\usepackage{times}
\usepackage{amsmath,amssymb,amsthm}
\usepackage{graphicx}
\usepackage{subfigure}
\usepackage{hyperref}
\usepackage{url}
\usepackage{pbox}

\newcommand{\inner}[2]{\left\langle #1, #2 \right\rangle}

\newcommand{\E}{\ensuremath{\mathbb{E}}}

\newcommand{\norm}[1]{\left\lVert{#1}\right\rVert}
\newcommand{\abs}[1]{\left\lvert{#1}\right\rvert}
\newcommand{\R}{\mathbb{R}}

\mathchardef\hyphen="2D

\usepackage{times}
\usepackage{amsmath,amssymb,amsthm}
\usepackage{graphicx}
\usepackage{subfigure}
\usepackage{hyperref}
\usepackage{url}
\usepackage{pbox}

\newtheorem{thm}{Theorem}
\newtheorem{lem}{Lemma}

\newcommand{\calH}{\mathcal{H}}

\newcommand{\calK}{\mathcal{K}}

\newcommand{\calM}{\mathcal{M}}

\newcommand{\calX}{\mathcal{X}}

\newcommand{\mathR}{\mathbb{R}}

\newcommand{\vecw}{\mathbf{w}}
\newcommand{\vecx}{\mathbf{x}}

\newcommand{\vecz}{\mathbf{z}}

\newcommand{\removed}[1]{}
\newcommand{\natinote}[1]{{\Large [[ {\color{red} {#1} -- Nati} ]]}}

\newcommand{\vecW}{\mathbf{w}}

\newcommand{\margin}{\text{margin}}
\newcommand{\hatl}{\widehat{L}}
\newcommand{\eps}{\boldsymbol{\nu}}
\newcommand{\err}{Err}
\newcommand{\error}{\epsilon}
\newcommand{\hatD}{\widehat{D}}

\usepackage[final]{nips_2017}

\usepackage[utf8]{inputenc} % allow utf-8 input
\usepackage[T1]{fontenc}    % use 8-bit T1 fonts
\usepackage{hyperref}       % hyperlinks
\usepackage{url}            % simple URL typesetting
\usepackage{booktabs}       % professional-quality tables
\usepackage{amsfonts}       % blackboard math symbols
\usepackage{nicefrac}       % compact symbols for 1/2, etc.
\usepackage{microtype}      % microtypography
\usepackage{graphicx}

\newcommand{\figdir}{.}

\title{Exploring Generalization in Deep Learning}

\author{
  Behnam Neyshabur,\;Srinadh Bhojanapalli,\;David McAllester,\;Nathan Srebro\\
  Toyota Technological Institute at Chicago\\
  \texttt{\{bneyshabur, srinadh, mcallester, nati\}@ttic.edu} \\
}

\begin{document}

\maketitle

\begin{abstract}
With a goal of understanding what drives generalization in deep networks, we consider several recently suggested explanations, including norm-based control, sharpness and robustness. We study how these measures can ensure generalization, highlighting the importance of scale normalization, and making a connection between sharpness and PAC-Bayes theory.  We then investigate how well the measures explain different observed phenomena.
\end{abstract}

\section{Introduction}
Learning with deep neural networks has enjoyed huge empirical success in recent years across a wide variety of tasks. Despite being a complex, non-convex optimization problem, simple methods such as stochastic gradient descent (SGD) are able to recover good solutions that minimize the training error. More surprisingly, the networks learned this way exhibit good generalization behavior, even when the number of parameters is significantly larger than the amount of training data \cite{neyshabur15b,zhang2017understanding}.

In such an over parametrized setting, the objective has multiple global minima, all minimize the training error, but many of them do not generalize well.  %I don't know of any paper that has shown this explicitly
Hence, just minimizing the training error is not sufficient for learning: picking the wrong global minima can lead to bad generalization behavior. In such situations, generalization behavior depends implicitly on the algorithm used to minimize the training error. Different algorithmic choices for optimization such as the initialization, update rules, learning rate, and stopping condition, will lead to different global minima with different generalization behavior \cite{chaudhari2016entropy, keskar2016large, NeySalSre15}. For example, \citet{NeySalSre15} introduced Path-SGD, an optimization algorithm that is invariant to rescaling of weights, and showed better generalization behavior over SGD for both feedforward and recurrent neural networks \cite{NeySalSre15,neyshabur2016path}. \citet{keskar2016large} noticed that the solutions found by stochastic gradient descent with large batch sizes generalizes worse than the one found with smaller batch sizes, and \citet{hardt2015train} discuss how stochastic gradient descent ensures uniform stability, thereby helping generalization for convex objectives.

% such as the initialization, update rules, learning rate, and stopping condition, that lead to finding different global minima 

What is the bias introduced by these algorithmic choices for neural
networks?  What ensures generalization in neural networks?  What is
the relevant notion of complexity or capacity control?

As mentioned above, simply accounting for complexity in terms of the
number of parameters, or any measure which is uniform across all
functions representable by a given architecture, is not sufficient to
explain the generalization ability of neural networks trained in
practice.  For linear models, norms and margin-based measures, and not
the number of parameters, are commonly used for capacity control~
\cite{bartlett2002rademacher, evgeniou2000regularization,
  smola1998connection}. Also norms such as the trace norm and max norm
are considered as sensible inductive biases in matrix factorization
and are often more appropriate than parameter-counting measures such
as the rank~\cite{srebro2005rank,srebro2005maximum}. In a similar
spirit, \citet{bartlett1998sample} and later \citet{neyshabur15b}
suggested different norms of network parameters to measure the
capacity of neural networks.  In a different line of work,
\citet{keskar2016large} suggested ``sharpness'' (robustness of the
training error to perturbations in the parameters) as a complexity
measure for neural networks.  Others, including
\citet{langford2001not} and more recently
\citet{dziugaite2017computing}, propose a PAC-Bayes analysis.

What makes a complexity measure appropriate for explaining
generalization in deep learning?  First, an appropriate complexity
measure must be sufficient in ensuring generalization.  Second,
networks learned in practice should be of low complexity under this
measure.  This can happen if our optimization algorithms bias us
toward lower complexity models under this measure {\em and} it is
possible to capture real data using networks of low complexity.  In
particular, the complexity measure should help explain several
recently observed empirical phenomena that are not explained by a
uniform notion of complexity:
\begin{itemize}
\item It is possible to obtain zero training error on random labels
  using the same architecture for which training with real labels
  leads to good generalization~\cite{zhang2017understanding}.  We
  would expect the networks learned using real labels (and which
  generalizes well) to have much lower complexity, under the suggested
  measure, than those learned using random labels (and which obviously
  do not generalize well).
\item Increasing the number of hidden units, thereby increasing the
  number of parameters, can lead to a decrease in generalization error
  even when the training error does not decrease~\cite{neyshabur15b}.
  We would expect to see the complexity measure decrease as we
  increase the number of hidden units.
\item When training the same architecture, with the same training set,
  using two different optimization methods (or different algorithmic
  or parameter choices), one method results in better generalization
  even though both lead to zero training
  error~\cite{NeySalSre15,keskar2016large}.  We would expect to see a
  correlation between the complexity measure and generalization
  ability among zero-training error models.
\end{itemize}
In this paper we examine different complexity measures that have
recently been suggested, or could be considered, in explaining
generalization in deep learning.  In light of the above, we evaluate
the measures based on their ability to theoretically guarantee
generalization, and their empirical ability to explain the above
phenomena.  Studying how each measure can guarantee generalization also
let us better understand how it should be computed and compared when
trying to explain the empirical phenomena.

We investigate complexity measures including norms, robustness and
sharpness of the network.  We emphasize in our theoretical and
empirical study the importance of relating the scale of the parameters
and the scale of the output of the network, e.g.~by relating norm and
margin.  In this light, we discuss how sharpness by itself is not
sufficient for ensuring generalization, but can be combined, through
PAC-Bayes analysis, with the norm of the weights to obtain an
appropriate complexity measure.  The role of sharpness in
PAC-Bayesian analysis of neural networks was also recently noted by
\citet{dziugaite2017computing}, who used numerical techniques to
numerically optimize the overall PAC-Bayes bound---here we emphasize
the distinct role of sharpness as a balance for norm.

\removed{This connection between sharpness and the PAC-Bayesian framework was also noticed in a prior work by \citet{dziugaite2017computing}, who optimize
the PAC-Bayes generalization bound over a family of multivariate
Gaussian distributions, extending the work of \citet{langford2001not}.
They show that the optimized PAC-Bayes bounds are numerically
non-vacuous for feedforward networks trained on a binary classification
variant of MNIST dataset.}

In order to further understand the
significance of sharpness in deep learning, and how its relationship
to margin deviates for that found in linear models, we also establish,
in Section \ref{sec:pac_bayes}, sufficient conditions on the network
that {\em provably} ensures small sharpness.

\subsection*{Notation}
Let $f_\vecw(\vecx)$ be the function computed by a $d$ layer
feed-forward network with parameters $\vecw$ and Rectified Linear Unit
(ReLU) activations, $f_\vecw(\vecx) =
W_d\,\phi(W_{d-1}\,\phi(....\phi(W_1 \vecx ) ) )$ where
$\phi(z)=\max\{0,z\}$. For a given $\vecx \in \R^n$, let
$D^{\vecx,\vecw}_i$ denote the diagonal $\{1,0\}$ matrix corresponding
to activation in layer $i$. To simplify the presentation we drop the
$\vecx$ superscript and use $D_i$ instead. We can therefore write
$f_\vecw(\vecx) = W_d\,D_{d-1}\,W_{d-1}\,\cdots \,D_1\,W_1\,\vecx =
W_d\left(\Pi_{i=1}^{d-1} D_{i}W_i \right)\vecx$ where we drop the
${\vecx,\vecw}$ superscript from $D_i^{\vecx,\vecw}$ and use $D_i$
instead but remember that $D_i$ depends on $\vecx$ and the parameters $W_j$ for any $j\leq i$.

\removed{
\footnote{Convolutional networks can also be expressed in a similar way but the notation becomes more involved due to max-pooling layers and weight sharing.}
}

Let $h_i$ be the number of nodes in layer $i$, with $h_0
=n$. Therefore, for any layer $i$, we have $W_i\in R^{h_{i}\times
  h_{i-1}}$. Given any input $x$, the loss of the prediction by the
function $f_\vecw$ is then given by $\ell(\vecw,\vecx)$. We also
denote by $L(\vecW)$ the expected loss and by $\hatl(\vecW)$ the
empirical loss over the training set. For any integer $k$, $[k]$
denotes the set $\{ 1, 2, \cdots, k\}$. Finally, $\norm{.}_F$,
$\norm{.}_2$, $\norm{.}_1$, $\norm{.}_\infty$ denote Frobenius norm,
the spectral norm, element-wise $\ell_1$-norm and element-wise
$\ell_\infty$ norm respectively.

\section{Generalization and Capacity Control in Deep Learning}\label{sec:summary}
In this section, we discuss complexity measures that have been
suggested, or could be used for capacity control in neural networks.
We discuss advantages and weaknesses of each of these complexity
measures and examine their abilities to explain the observed
generalization phenomena in deep learning.

We consider the statistical {\em capacity} of a model class in terms
of the number of examples required to ensure {\em generalization},
i.e.~that the population (or test error) is close to the training
error, even when minimizing the training error.  This also roughly
corresponds to the maximum number of examples on which one can obtain
small training error even with random labels.

Given a model class $\calH$, such as all the functions representable by
some feedforward or convolutional networks, one can consider the
capacity of the entire class $\calH$---this corresponds to learning
with a uniform ``prior'' or notion of complexity over all models in
the class.  Alternatively, we can also consider some {\em complexity
  measure}, which we take as a mapping that assigns a non-negative
number to every hypothesis in the class - $\calM: \{ \calH, S \}
\rightarrow \mathR^+$, where $S$ is the training set.  It is then sufficient to consider the capacity
of the restricted class $\calH_{\calM,\alpha}=\{h: h\in \calH,
\calM(h) \leq \alpha\}$ for a given $\alpha \geq 0$.  One can then
ensure generalization of a learned hypothesis $h$ in terms of the
capacity of $\calH_{\calM,\calM(h)}$.  Having a good hypothesis with
low complexity, and being biased toward low complexity (in terms of
$\calM$) can then be sufficient for learning, even if the capacity of
the entire $\calH$ is high.  And if we are indeed relying on $\calM$
for ensuring generalization (and in particular, biasing toward models
with lower complexity under $\calM$), we would expect a learned $h$
with lower value of $\calM(h)$ to generalize better.

For some of the measures discussed, we allow $\calM$ to depend also on
the training set.  If this is done carefully, we can still ensure
generalization for the restricted class $\calH_{\calM,\alpha}$.

We will consider several possible complexity measures.  For each
candidate measure, we first investigate whether it is sufficient for
generalization, and analyze the capacity of $\calH_{\calM,\alpha}$.
Understanding the capacity corresponding to different complexity
measures also allows us to relate between different measures and
provides guidance as to what and how we should measure: From the above
discussion, it is clear that any monotone transformation of a
complexity measures leads to an equivalent notion of complexity.
Furthermore, complexity is meaningful only in the context of a
specific hypothesis class $\calH$, e.g.~specific architecture or
network size.  The capacity, as we consider it (in units of sample
complexity), provides a yardstick by which to measure complexity (we
should be clear though, that we are vague regarding the scaling of the
generalization error itself, and only consider the scaling in terms of
complexity and model class, thus we obtain only a very crude yardstick
sufficient for investigating trends and relative phenomena, not a
quantitative yardstick).

\subsection{Network Size}
For any model, if its parameters have finite precision, its capacity
is linear in the total number of parameters. Even without making an
assumption on the precision of parameters, the VC dimension of
feedforward networks can be bounded in terms of the number of
parameters 
$\text{dim}(\vecW)$\cite{anthony2009neural,bartlett1998sample,bartlett1998almost,
  shalev2014understanding}. In particular, \citet{bartlet2017} and
\citet{harvey2017nearly}, following \citet{bartlett1998almost}, give the following tight
(up to logarithmic factors) bound on the VC
dimension and hence capacity of feedforward networks with ReLU activations:
\begin{equation}
\text{VC-dim} = \tilde{O}(d * \text{dim}(\vecW))
\end{equation}
In the over-parametrized settings, where the number of parameters is
more than the number of samples, complexity measures that depend on
the total number of parameters are too weak and cannot explain the
generalization behavior.  Neural networks used in practice often have
significantly more parameters than samples, and indeed can perfectly
fit even random labels, obviously without
generalizing~\cite{zhang2017understanding}.  Moreover, measuring
complexity in terms of number of parameters cannot explain the
reduction in generalization error as the number of hidden units
increase \cite{neyshabur15b} (see also Figure
Figure~\ref{fig:hidden}).

%So we need to consider in addition, the complexity reducing effects of regularization, either explicit or implicit, through various training choices such as weight-decay, early stopping, dropout etc. Therefore, the main shortcoming of network size as capacity control parameter is the looseness due to ignoring the regularization effects.

%Network size bounds the capacity of unregularized models. However, in practice the training procedure involves several explicit or implicit regularizations such as weight-decay, early stopping, dropout etc. Maybe an even more crucial implicit regularization is by the optimization algorithm. Since large neural networks have very high capacity, there are many global optima for the loss function and the chance of getting to each global optima is not the same for a given optimization algorithm. Therefore, the main shortcoming of network size as capacity control parameter is the looseness due to ignoring regularization affects.

\subsection{Norms and Margins}\label{sec:margin}

Capacity of linear predictors can be controlled independent of the
number of parameters, e.g. through regularization of its $\ell_2$ norm.
Similar norm based complexity measures have also been established
for feedforward neural networks with ReLU activations. For
example, capacity can be bounded based on the $\ell_1$ norm of the
weights of hidden units in each layer, and is proportional to
$\prod_{i=1}^d \norm{W_i}^2_{1,\infty}$, where $\norm{W_i}_{1,\infty}$
is the maximum over hidden units in layer $i$ of the $\ell_1$ norm of
incoming weights to the hidden unit \cite{bartlett2002rademacher}.
More generally \citet{NeyTomSre15} considered group norms $\ell_{p,
  q}$ corresponding to $\ell_q$ norm over hidden units of $\ell_p$
norm of incoming weights to the hidden unit. This includes
$\ell_{2,2}$ which is equivalent to the Frobenius norm where the
capacity of the network is proportional to $\prod_{i=1}^d
\norm{W_i}^2_{F}$. They further motivated a complexity measure that is
invariant to node-wise rescaling reparametrization
\footnote{Node-rescaling can be defined as a sequence of
  reparametrizations, each of which corresponds to multiplying
  incoming weights and dividing outgoing weights of a hidden unit by a
  positive scalar $\alpha$. The resulting network computes the same
  function as the network before the reparametrization.}, suggesting
$\ell_p$ path norms which is the minimum over all node-wise rescalings
of $\prod_{i=1}^d \norm{W_i}_{p,\infty}$ and is equal to $\ell_p$ norm
of a vector with coordinates each of which is the product of weights
along a path from an input node to an output node in the network.

\removed{
Capacity can also be bounded in terms of $\ell_1$ and $\ell_2$ path norms where $\ell_p$ path-norm is defined as $\ell_p$ norm over all paths in the network of product of weights along the path~\cite{NeyTomSre15}.}

Capacity control in terms of norm, when using a zero/one loss
(i.e.~counting errors) requires us in addition to account for scaling
of the output of the neural networks, as the loss is insensitive to
this scaling but the norm only makes sense in the context of such
scaling.  For example, dividing all the weights by the same number will
scale down the output of the network but does not change the $0/1$
loss, and hence it is possible to get a network with arbitrary small
norm and the same $0/1$ loss.  Using a scale sensitive losses, such as the
cross entropy loss, does address this issue (if the outputs are scaled
down toward zero, the loss becomes trivially bad), and one can obtain
generalization guarantees in terms of norm and the cross entropy loss.

However, we should be careful when comparing the norms of different
models learned by minimizing the cross entropy loss, in particular
when the training error goes to zero.  When the training error goes to
zero, in order to push the cross entropy loss (or any other positive
loss that diminish at infinity) to zero, the outputs of the network
must go to infinity, and thus the norm of the weights (under any
norm) should also go to infinity.  This means that minimizing the
cross entropy loss will drive the norm toward infinity.  In practice,
the search is terminated at some finite time, resulting in large, but
finite norm.  But the value of this norm is mostly an indication of
how far the optimization is allowed to progress---using a stricter
stopping criteria (or higher allowed number of iterations) would yield
higher norm.  In particular, comparing the norms of models found using
different optimization approaches is meaningless, as they would all go
toward infinity.

Instead, to meaningfully compare norms of the network, we should
explicitly take into account the scaling of the outputs of the
network. One way this can be done, when the training error is indeed
zero, is to consider the ``margin'' of the predictions in addition to
the norms of the parameters.  We refer to the margin for a single data
point $x$ as the difference between the score of the correct label and
the maximum score of other labels, i.e.
\begin{equation}
  \label{eq:margin}
 f_\vecw(\vecx)[y_\text{true}] -
\max_{y\neq y_\text{true}} f_\vecw(\vecx)[y] 
\end{equation}
In order to measure scale over an entire training set, one simple
approach is to consider the ``hard margin'', which is the minimum
margin among all training points.  However, this definition is very
sensitive to extreme points as well as to the size of the training
set.  We consider instead a more robust notion that allows a small
portion of data points to violate the margin. For a given training set
and small value $\epsilon>0$, we define the margin $\gamma_\margin$ as
the lowest value of $\gamma$ such that $\lceil \epsilon m \rceil$ data
point have margin lower than $\gamma$ where $m$ is the size of the
training set.  We found empirically that the qualitative and relative
nature of our empirical results is almost unaffected by reasonable
choices of $\epsilon$ (e.g.~between $0.001$ and $0.1$).

\removed{
Equivalently, margin for a given point can be viewed as the amount of perturbation to the output that the predictor can tolerate without changing $0/1$ loss of the prediction~\footnote{The margin value for the former definition is actually twice the margin in the latter but since every margin value will be multiplied by 2, these definitions are equivalent.}. For a given set, margin can be seen as the maximum perturbation of the output scores of the predictor that only changes the loss by at most $\error$:
\begin{equation}
\gamma_{\text{out}} = \max_{\gamma \geq 0} \enskip\gamma\qquad\text{s.t.}\qquad  \error \geq \max_{\norm{\eps}_\infty \leq \gamma}\frac{1}{m} \sum_{i=1}^m\ell( f_\vecw(x_i)+\eps_i, y_i) - \ell( f_\vecw(x_i), y_i)
\end{equation}

%Note that talking about norms without taking scaling into account is meaningless since these bounds are for scale-sensitive loss functions such as cross-entropy and if the predictions are correct, the scale sensitive loss can be improved by simply scaling the outputs of the networks. In order to account for the scaling, instead of only considering the norms, we have to consider the ratio of the norm divided by a notion of margin that captures the scaling. One possible notion is the maximum perturbation of the output scores of the predictor that only changes the loss by at most $\error_S$
where $\eps$ is the perturbation matrix whose columns are perturbations on the output of the network for each data point and $\|.\|_\infty$ is simply the maximum element of this matrix.}

\removed{We will combine the margin with norm/measure of the parameters to ensure that the resulting complexity measure is invariant to the scaling of the output. To build a scale invariant measure, we need the norm/measure to change at the same rate as the output of the function $f_\vecw(x)$. Given such a measure $\norm{.}_M$, we can consider $\frac{\norm{.}_M}{\gamma_{\margin}}$, which is scale invariant, as a complexity measure for the feedforward networks with ReLU activations. }

\removed{
{\color{red} The paragraph below is not clear}

Another way to understand the reason for rescaling by the margin is through the relationship between losses. Suppose we use $\gamma_\margin$ as threshold for prediction, i.e. if Scaling the network output by $\frac{1}{\gamma_\margin}$ results in unit margin and in that case a margin based loss such as hinge loss can upper bound 0/1 loss. Hence the bound on hinge loss can be applied to get a bound on 0/1 loss. Therefore, we need to scale the network output by $\frac{1}{\gamma_\margin}$. This property holds for the measures explained above since scaling up the weights by a constant factor $c$ will scale up both margin and norm by factor $c^d$. 
}

The norm-based measures we investigate in this work and their
corresponding capacity bounds are as follows \footnote{We have dropped the term that only depend on the norm of the input. The bounds based on $\ell_2$-path norm and spectral norm can be derived directly from the those based on $\ell_1$-path norm and $\ell_2$ norm respectively. Without further conditions on weights, exponential dependence on depth is tight but the $4^d$ dependence might be loose~\cite{NeyTomSre15}. We will also discuss a rather loose bound on the capacity based on the spectral norm in Section \ref{subsec:lipschitz}.}:

\begin{itemize}
\item $\ell_2$ norm with capacity proportional to $\frac{1}{\gamma_{\margin}^2}\prod_{i=1}^d 4\norm{W_i}^2_F$~\cite{NeyTomSre15}.
\item $\ell_1$-path norm with capacity proportional to $\frac{1}{\gamma_{\margin}^2}\left(\sum_{j \in \prod_{k=0}^d[h_k]}\abs{\prod_{i=1}^d 2W_i[j_i,j_{i-1}]}\right)^2$\cite{bartlett2002rademacher,NeyTomSre15}.
\item $\ell_2$-path norm with capacity proportional to $\frac{1}{\gamma_{\margin}^2}\sum_{j \in \prod_{k=0}^d[h_k]}\prod_{i=1}^d 4h_iW_i^2[j_i,j_{i-1}]$.
\item spectral norm with capacity proportional to $\frac{1}{\gamma_{\margin}^2}\prod_{i=1}^d h_i\norm{W_i}^2_2$.
\end{itemize}
where $\prod_{k=0}^d[h_k]$ is the Cartesian product over sets $[h_k]$. The above bounds indicate that capacity can be bounded in terms of either $\ell_2$-norm or $\ell_1$-path norm independent of number of parameters. The $\ell_2$-path norm dependence on the number of hidden units in each layer is unavoidable. However, it is not clear that the dependence on the number of parameters is needed for the bound based on the spectral norm. 
\removed{
\natinote{Need to add discussion about which ones can be used to bound
  capacity in a way that's independent, or weakly dependent, on number of parameters, and
  which can't, and which we are not sure.}
 }
%Table \ref{fig:measure} indicates the measures that will be investigated in this work and the capacity of the network using that measure.

%\begin{table}
%\small
%\setlength\tabcolsep{4pt}
%\begin{tabular}[tb]{|c|c|c|c|}
%\hline
%Measure & Capacity prop. to &  Measure
%& Capacity prop. to\\
%\hline
%$\ell_2$ norm&$\prod_{i=1}^d 2\norm{W_i}_F/\gamma_{\text{out}}$&  $\ell_1$-path norm& $\sum_{j_k \in [h]}\abs{\prod_{i=1}^d 2W_i[j_i,j_{i-1}]}/\gamma_{\text{out}}$\\
%\hline
%spectral norm& $\prod_{i=1}^d \norm{W_i}_2 \calN(\gamma/2,\calX,\norm{.}_2)$& $\ell_2$-path norm
%& $2^dh^{d/2} \left(\sum_{j_k \in [h]}\prod_{i=1}^d W_i^2[j_i,j_{i-1}]\right)^{1/2}/\gamma_{\text{out}}$\\
%\hline
%\end{tabular}
%  \caption{Capacity of neural networks with respect to the norms investigated in the paper.}
%  \label{fig:measure}
%\end{table}

\begin{figure}[t]
\centering
\includegraphics[width=.245\textwidth]{\figdir/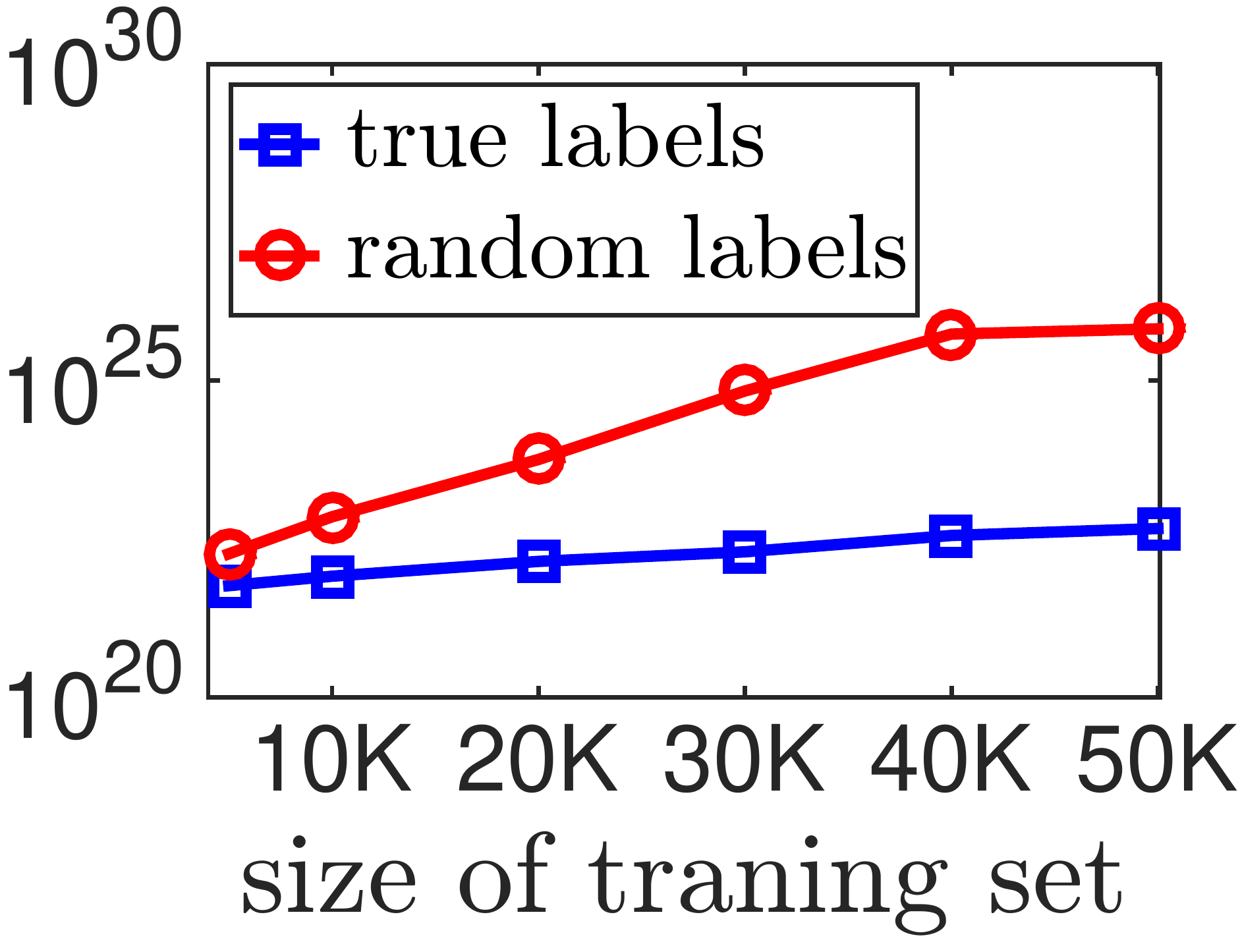}
\includegraphics[width=.245\textwidth]{\figdir/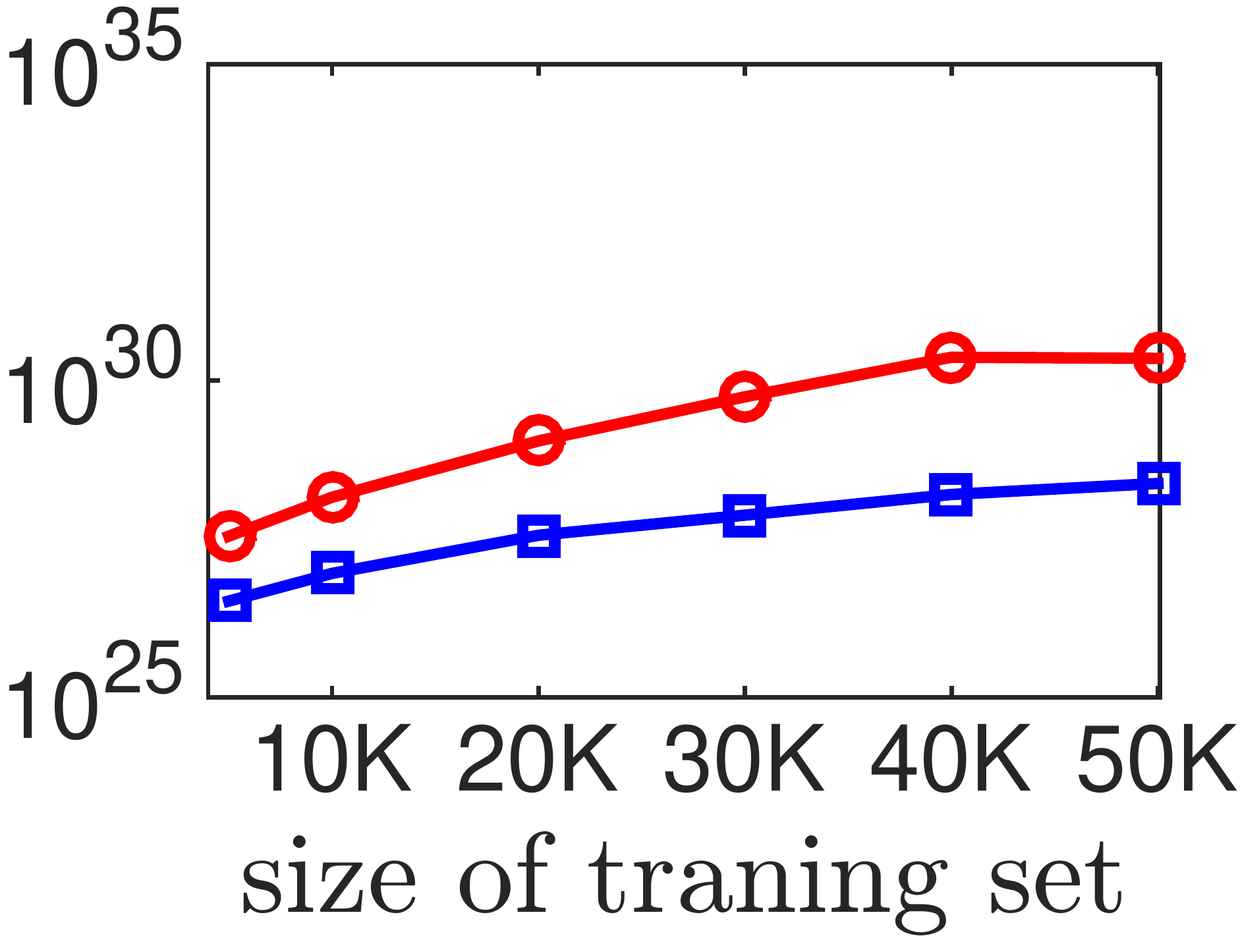}
\includegraphics[width=.245\textwidth]{\figdir/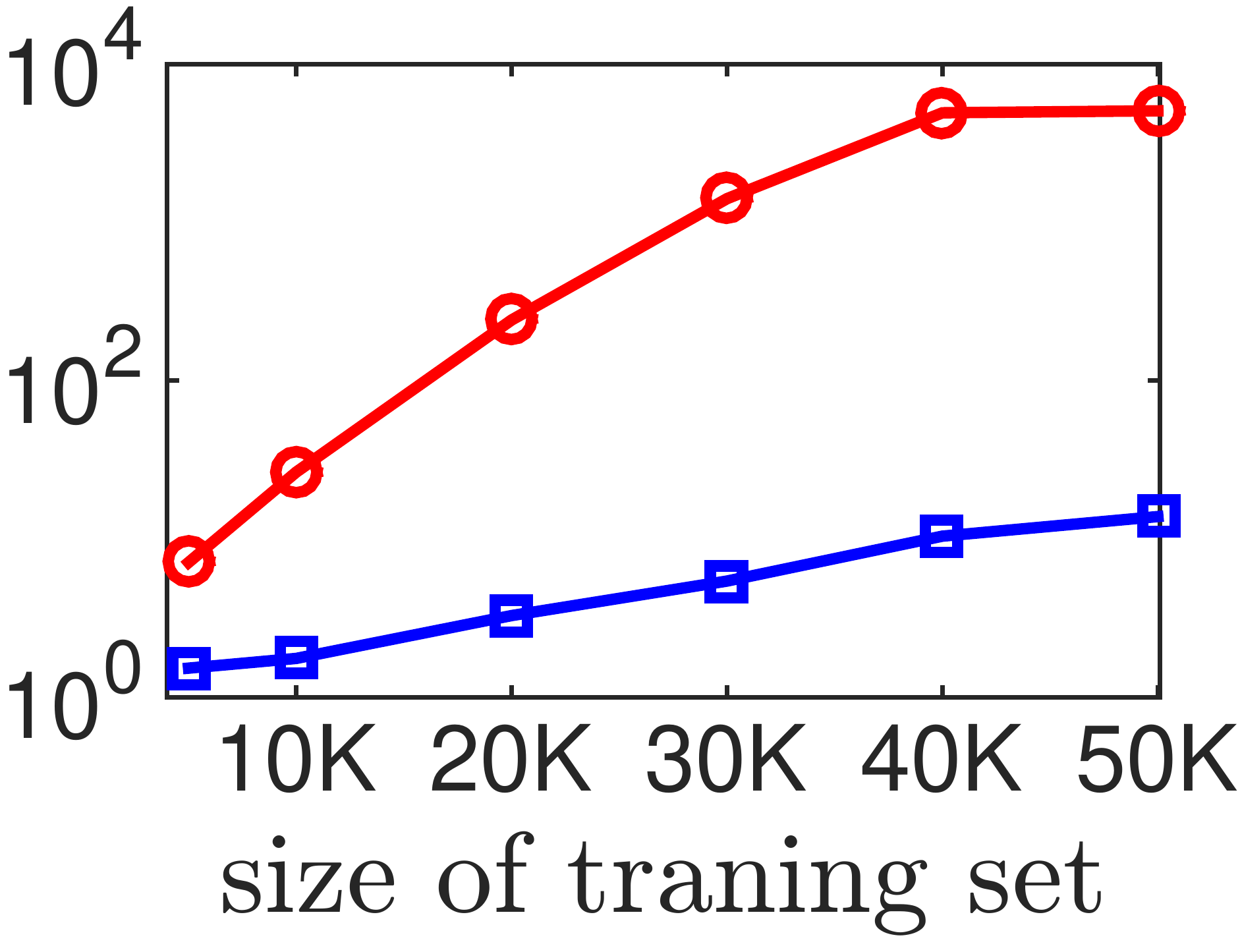}
\includegraphics[width=.245\textwidth]{\figdir/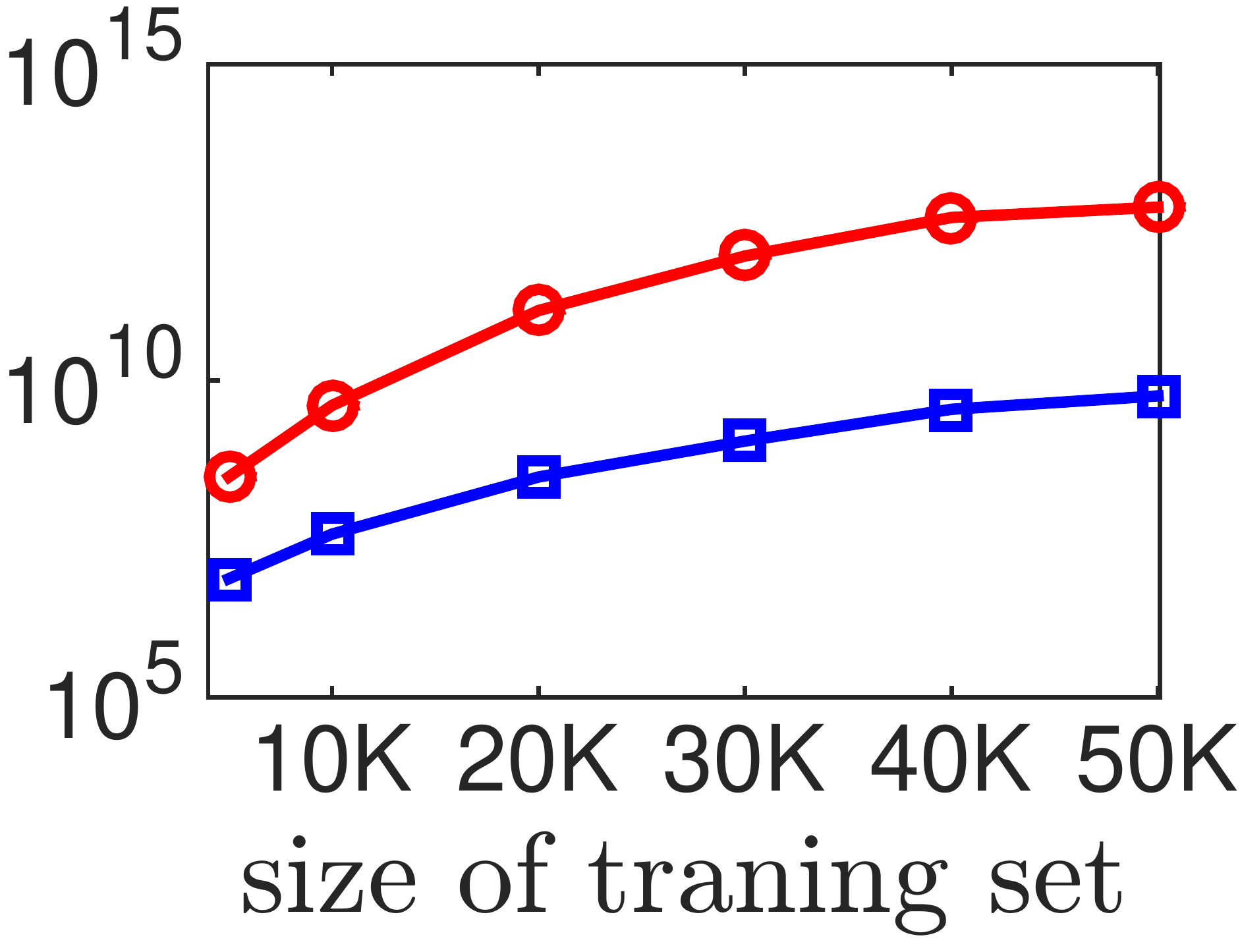}
\begin{picture}(0,0)(0,0)
{\small \put(-162, 85){$\ell_2$ norm}\put(-70, 85){$\ell_1$-path norm}\put(28, 85){$\ell_2$-path norm}\put(125, 85){spectral norm}}
\end{picture}
\caption{\small Comparing different complexity measures on a VGG network trained on subsets of CIFAR10 dataset with true (blue line) or random (red line) labels. We plot norm divided by margin to avoid scaling issues (see Section~\ref{sec:summary}), where for each complexity measure, we drop the terms that only depend on depth or number of hidden units; e.g. for $\ell_2$-path norm we plot $\gamma_{\margin}^{-2}\sum_{j \in \prod_{k=0}^d[h_k]}\prod_{i=1}^d W_i^2[j_i,j_{i-1}]$.We also set the margin over training set $S$ to be $5^{th}$-percentile of the margins of the data points in $S$, i.e. $\text{Prc}_5\left\{f_\vecw(x_i)[y_i] - \max_{y\neq y_i} f_\vecw(\vecx)[y] | (x_i,y_i)\in S\right\}$. In all experiments, the training error of the learned network is zero. The plots indicate that these measures can explain the generalization as the complexity of model learned with random labels is always higher than the one learned with true labels. Furthermore, the gap between the complexity of models learned with true and random labels increases as we increase the size of the training set.}
\label{fig:norm-true-random}
\end{figure}
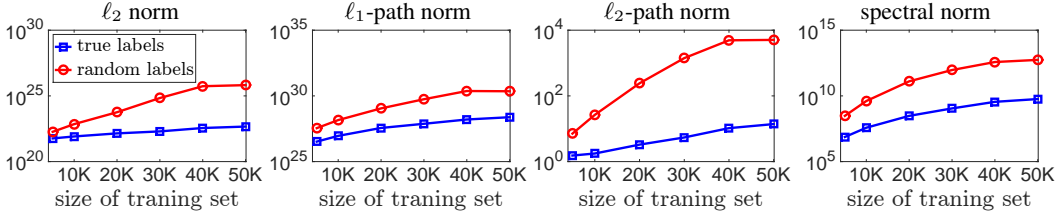
As an initial empirical investigation of the appropriateness of the
different complexity measures, we compared the complexity (under each
of the above measures) of models trained on true versus random labels.
We would expect to see two phenomena: first, the complexity of models
trained on true labels should be substantially lower than those
trained on random labels, corresponding to their better generalization
ability.  Second, when training on random labels, we expect capacity
to increase almost linearly with the number of training examples, since
every extra example requires new capacity in order to fit it's random
label.  However, when training on true labels we expect the model to
capture the true functional dependence between input and output and
thus fitting more training examples should only require small
increases in the capacity of the network.  The results are reported in
Figure \ref{fig:norm-true-random}.  We indeed observe a gap between
the complexity of models learned on real and random labels for all
four norms, with the difference in increase in capacity between true
and random labels being most pronounced for the $\ell_2$ norm and
$\ell_2$-path norm.

In Section \ref{sec:empirical} we present further empirical
investigations of the appropriateness of these complexity measures to
explaining other phenomena.

\subsection{Lipschitz Continuity and Robustness}\label{subsec:lipschitz}
The measures/norms we discussed so far also control the Lipschitz
constant of the network with respect to its input.  Is the capacity
control achieved through the bound on the Lipschitz constant?  Is bounding the
Lipschitz constant alone enough for generalization?  To answer these
questions, and in order to understand capacity control in terms of
Lipschitz continuity more broadly, we review here the relevant guarantees.

Given an input space $\calX$ and metric $\calM$, a function $f:\calX\rightarrow \R$ on a metric space $(\calX,\calM)$ is called
a Lipschitz function if there exists a constant $C_\calM$, such that $\abs{f(x)-f(y)}\leq C_\calM \calM(x,y)$.
\citet{luxburg2004distance} studied the capacity of functions with
bounded Lipschitz constant on metric space $(\calX,\calM)$ with a finite diameter $\text{diam}_\calM(\calX)=\sup_{x,y\in X} \calM(x,y)$ 
and showed that the capacity is proportional to $\left(\frac{C_\calM}{\gamma_{\margin}}\right)^n \text{diam}_\calM(\calX)$. This capacity bound is weak as it has an
exponential dependence on input size.

Another related approach is through algorithmic robustness as suggested by \citet{xu2012robustness}. Given $\epsilon>0$, the model $f_\vecw$ found by a learning algorithm is $K$ robust if $\calX$ can be partitioned into $K$ disjoint sets, denoted as $\{C_i\}_{i=1}^K$, such that for any pair $(\vecx,y)$ in the training set $\mathbf{s}$ ,\footnote{\citet{xu2012robustness} have defined the robustness as a property of learning algorithm given the model class and the training set. Here since we are focused on the learned model, we introduce it as a property of the model.}
\begin{equation}
\vecx,\vecz\in C_i \Rightarrow \abs{\ell(\vecw,\vecx)-\ell(\vecw,\vecz)}\leq \epsilon
\end{equation}
\citet{xu2012robustness} showed the capacity of a model class whose
models are $K$-robust scales as $K$. For the model class of functions
with bounded Lipschitz $C_{\norm{.}}$, $K$ is proportional to
${\frac{C_{\norm{.}}}{\gamma_{\margin}}}$-covering number of the
input domain $\calX$ under norm $\norm{.}$. However, the covering
number of the input domain can be exponential in the input dimension
and the capacity can still grow as
$\left(\frac{C_{\norm{.}}}{\gamma_{\margin}}\right)^n$~\footnote{Similar to margin-based bounds, we drop the term that depends on the diameter of the input space.}.

Returning to our original question, the $C_{\ell_\infty}$ and $C_{\ell_2}$
Lipschitz constants of the network can be bounded by $\prod_{i=1}^d
\norm{W_i}_{1,\infty}$ (hence $\ell_1$-path norm) and $\prod_{i=1}^d
\norm{W_i}_{2}$,
respectively~\cite{xu2012robustness,sokolic2016generalization}. This
will result in a very large capacity bound that scales as
$\left(\frac{\prod_{i=1}^d \norm{W_i}_{2}}{\gamma_\margin}\right)^n$,
which is exponential in both the input dimension and depth of the
network. This shows that simply bounding the Lipschitz constant of the
network is not enough to get a reasonable capacity control, and the
capacity bounds of the previous Section are not merely a consequence
of bounding the Lipschitz constant.

\subsection{Sharpness}
The notion of sharpness as a generalization measure was recently suggested by \citet{keskar2016large} and corresponds to robustness to adversarial perturbations on the parameter space:
\begin{equation}
\zeta_\alpha(\vecW) = \frac{\max_{\abs{\eps_i}\leq \alpha (\abs{\vecw_i}+\mathbf{1})}\hatl(f_{\vecw+\eps}) - \hatl(f_\vecw)}{1+\hatl(f_\vecw)} \simeq \max_{\abs{\eps_i}\leq \alpha (\abs{\vecw_i}+\mathbf{1})}\hatl(f_{\vecw+\eps}) - \hatl(f_\vecw),
\end{equation}
where the training error $\hatl(f_\vecw)$ is generally very small in the case of neural networks in practice, so we can simply drop it from the denominator without a significant change in the sharpness value. 

As we will explain below, sharpness defined this way does {\em not}
capture the generalization behavior. To see this, we first examine
whether sharpness can predict the generalization behavior for networks
trained on true vs random labels. In the left plot of
Figure~\ref{fig:sharpness-true-random}, we plot the sharpness for
networks trained on true vs random labels.  While sharpness correctly
predicts the generalization behavior for bigger networks, for networks
of smaller size, those trained on random labels have less sharpness
than the ones trained on true labels.  Furthermore sharpness defined
above depends on the scale of $\vecw$ and can be artificially
increased or decreased by changing the scale of the parameters.
Therefore, sharpness alone is not sufficient to control the capacity
of the network. 

Instead, we advocate viewing a related notion of expected sharpness in
the context of the PAC-Bayesian framework.  Viewed this way, it
becomes clear that sharpness controls only one of two relevant terms,
and must be balanced with some other measure such as norm.  Together,
sharpness and norm do provide capacity control and can explain many of
the observed phenomena.  This connection between sharpness and the
PAC-Bayes framework was also recently noted by \citet{dziugaite2017computing}. 

The PAC-Bayesian framework~\cite{mcallester1998some,mcallester1999pac}
provides guarantees on the expected error of a randomized predictor
(hypothesis), drawn form a distribution denoted $\mathcal{Q}$ and
sometimes referred to as a ``posterior'' (although it need {\em not}
be the Bayesian posterior), that depends on the training data. Let $f_\vecw$
be any predictor (not necessarily a neural network) learned from training data.
We consider a distribution $\mathcal{Q}$ over predictors with
weights of the form $\vecw+\eps$, where $\vecw$ is a single predictor
learned from the training set, and $\eps$ is a random variable.  Then,
given a ``prior'' distribution $P$ over the hypothesis that is
independent of the training data, with probability at least $1-\delta$
over the draw of the training data, the expected error of
$f_{\vecw+\eps}$ can be bounded as follows~\cite{mcallester2003simplified}:
\begin{small}
\begin{equation}\label{eq:pac-bayes-general}
\E_{\eps} [L(f_{\vecw+\eps})] \leq \E_{\eps} [\hatl(f_{\vecw+\eps})] +\sqrt{\E_{\eps} [\hatl(f_{\vecw+\eps})]\calK}+\calK
\end{equation}
\end{small}
where $\calK=\frac{2\left(KL\left(\vecw+\eps\|P\right)+\ln\frac{2m}{\delta}\right)}{m-1}$. When the training loss $\E_{\eps} [\hatl(f_{\vecw+\eps})]$ is smaller than $\calK$, then the last term dominates. This is often the case for neural networks with small enough perturbation. One can also get the the following weaker bound:
\begin{small}
\begin{equation}\label{eq:pac-bayes-simple}
\E_{\eps} [L(f_{\vecw+\eps})] \leq \E_{\eps} [\hatl(f_{\vecw+\eps})] +4\sqrt{\frac{\left(KL\left(\vecw+\eps\|P\right)+\ln\frac{2m}{\delta}\right)}{m}}
\end{equation}
\end{small}
The above inequality clearly holds for $\calK\geq 1$ and for $\calK<1$ it can be derived from Equation~\eqref{eq:pac-bayes-general} by upper bounding the loss in the second term by $1$. We can rewrite the above bound as follows:
\begin{small}
\begin{equation}
\E_{\eps} [L(f_{\vecw+\eps})] \leq \hatl(f_\vecw) + \underbrace{\E_{\eps} [\hatl(f_{\vecw+\eps})] -\hatl(f_\vecw)}_{\text{expected sharpness}} +4\sqrt{\frac{1}{m}\left(KL\left(\vecw+\eps\|P\right)+\ln\frac{2m}{\delta}\right)}
\end{equation}
\label{eq:pacbayes}
\end{small}

\removed{
For a given norm $\norm{.}$ and maximum perturbation magnitude $\gamma_{\text{param}}$, sharpness ($\error_{\gamma_{\text{param}}}$) is defined as the maximum change in the loss with a perturbation of magnitude $\gamma_{\text{param}}$. However, it seems that rather than measuring the sharpness for a fixed $\gamma_{\text{param}}$, it might be more useful to fix a small $\error$ and find the maximum $\gamma_{\text{param}}$ that changes the loss by at most $\error$:
\begin{equation}
\gamma_{\text{param}} = \max_{\gamma \geq 0} \enskip\gamma\qquad\text{s.t.}\qquad  \error \geq \max_{\norm{\eps} \leq \gamma}\frac{1}{m} \sum_{i=1}^m\ell( f_{W+\eps}(x_i), y_i) - \ell( f_\vecw(x_i), y_i)
\end{equation}

Sharpness is shown to correlate with generalization error of models trained with large batch sizes \cite{keskar2016large}. However, sharpness defined this way based on adversarial perturbation seems to be an unstable measure (see left panel in Figure~\ref{fig:sharpness-true-random}). Hence, we propose looking at {\it expected sharpness} of the network, where we measure sharpness by setting $\eps$ according to a distribution. Even more significantly, sharpness defined above depends on the scale of $W$ and can be artificially increased or decreased by changing the scale of the parameters. %This dependence on the scaling and parametrization was also noticed by \citet{dinh2017sharp}.

Interestingly, this notion of expected sharpness appears in the generalization error computation according to PAC-Bayesian framework. PAC-Bayesian framework is a popular way of thinking about generalization behavior, where the hypothesis (predictor) is chosen from a distribution $Q$, with mean $W$, $Q= W+\eps$. Then, given a prior $P$ over parameters, with probability at least $1-\delta$ over the draw of the samples,  the generalization error can be decomposed as follows:
\begin{small}
\begin{equation}
\E_{\eps} [L(f_{\vecw+\eps})] \leq \hatl(f_\vecw) + \underbrace{\E_{\eps} [\hatl(f_{W+\eps})] -\hatl(f_\vecw)}_{\text{sharpness}} +\sqrt{\frac{1}{m}\left(KL\left(W+\eps\|P\right)+\ln(\frac{1}{\delta})\right)}.
\end{equation}
\label{eq:pacbayes}
\end{small}
}

\begin{figure}[t]
\centering
\includegraphics[width=.32\textwidth]{\figdir/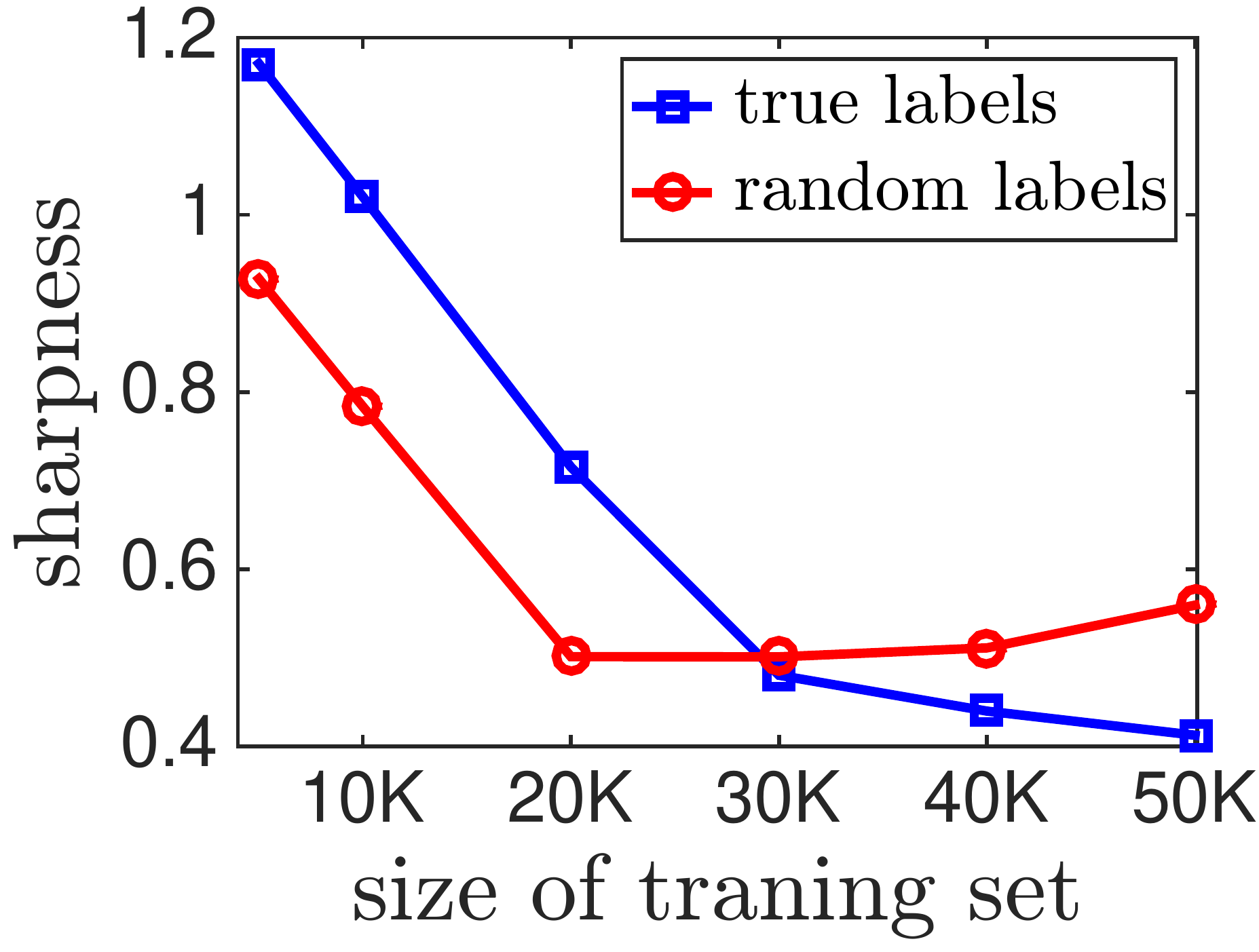}
\includegraphics[width=.32\textwidth]{\figdir/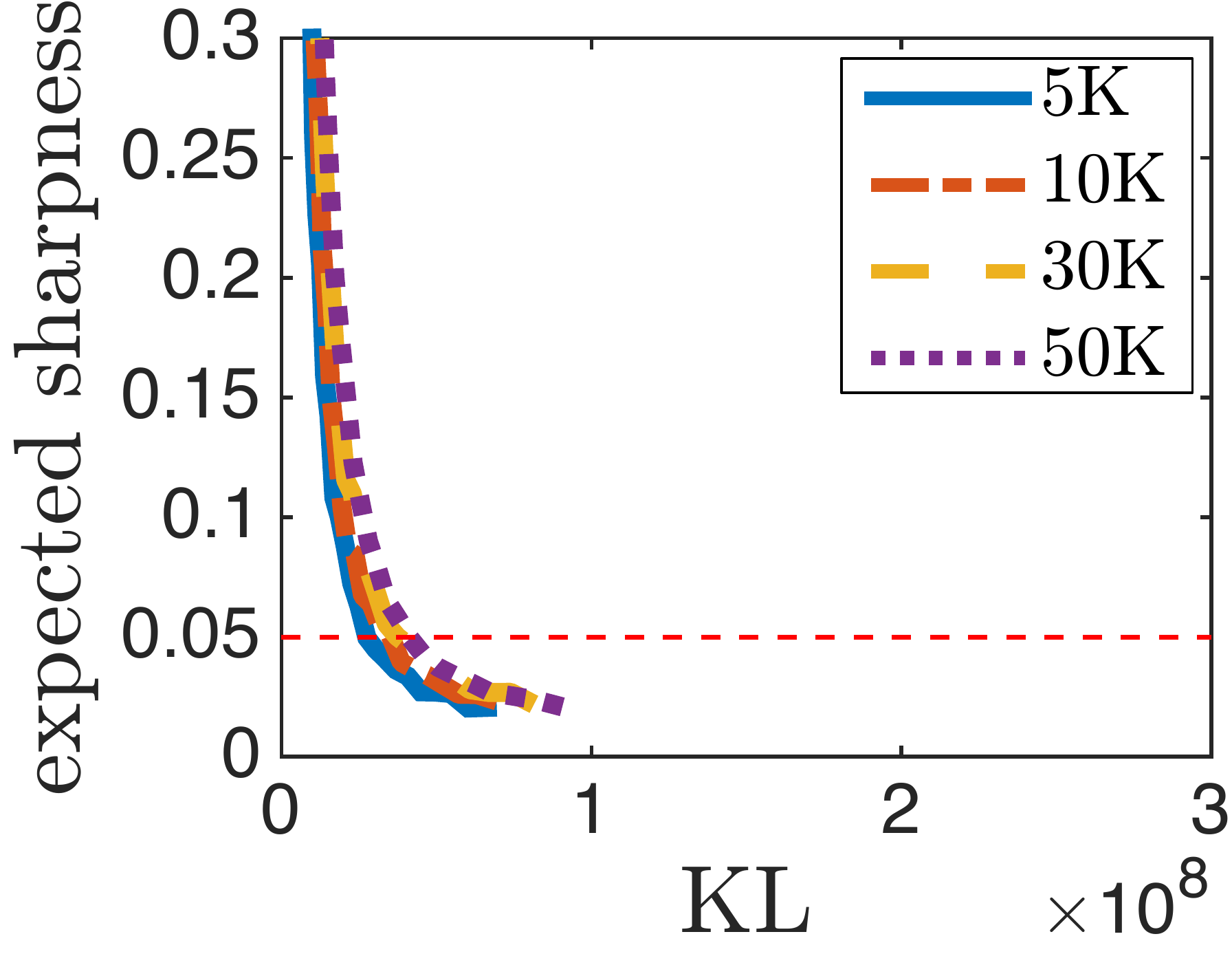}
\includegraphics[width=.32\textwidth]{\figdir/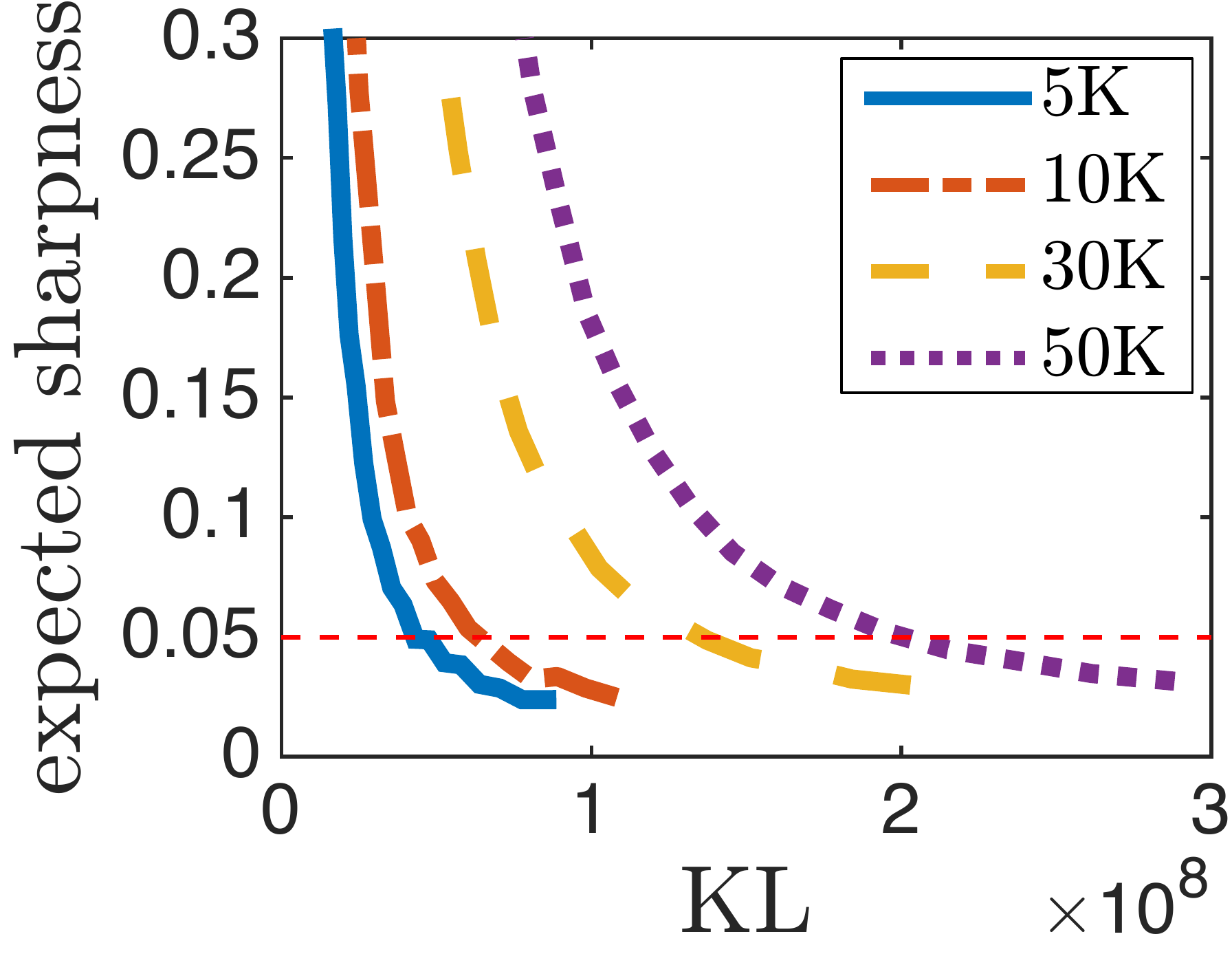}
\caption{\small Sharpness and PAC-Bayes measures on a VGG network
  trained on subsets of CIFAR10 dataset with true or random labels. In
  the left panel, we plot max sharpness, which we calculate as
  suggested by \citet{keskar2016large} where the perturbation for
  parameter $w_i$ has magnitude $5.10^{-4}(\abs{w_i}+1)$. The middle
  and right plots demonstrate the relationship between expected
  sharpness and KL divergence in PAC-Bayes analysis for true and
  random labels respectively. For PAC-Bayes plots, each point in the
  plot correspond to a choice of variable $\alpha$ where the standard
  deviation of the perturbation for the parameter $i$ is
  $\alpha(10\abs{w_i}+1)$. The corresponding $KL$ to each $\alpha$ is
  nothing but weighted $\ell_2$ norm where the weight for each
  parameter is the inverse of the standard deviation of the
  perturbation.}
\begin{picture}(0,0)(0,0)
{\small \put(-6, 196){true labels}\put(114, 196){random labels}}
\end{picture}
\label{fig:sharpness-true-random}
\end{figure}
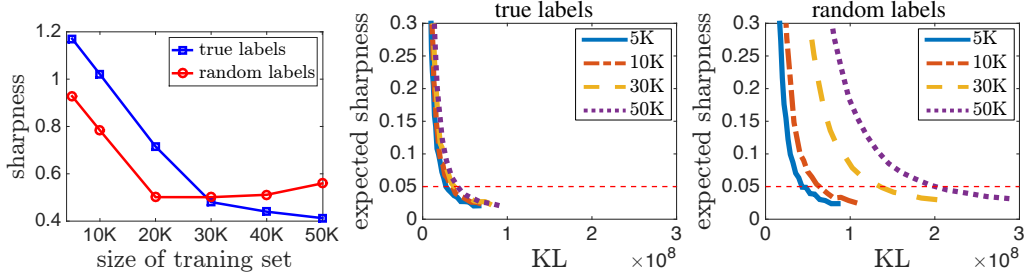

As we can see, the PAC-Bayes bound depends on two quantities - i) the
expected sharpness and ii) the Kullback Leibler (KL) divergence to the
``prior'' $P$.  The bound is valid for any distribution measure $P$,
any perturbation distribution $\eps$ and any method of choosing
$\vecw$ dependent on the training set.  A simple way to instantiate the bound
is to set $P$ to be a zero mean, $\sigma^2$ variance Gaussian distribution.
Choosing the perturbation $\eps$ to also be a zero mean spherical
Gaussian with variance $\sigma^2$ in every direction, yields the
following guarantee (w.p.~$1-\delta$ over the training set):
\begin{small}
\begin{equation}\label{eq:pacbayes2}
\E_{\eps \sim \mathcal{N}(0,\sigma)^n} [L(f_{\vecw+\eps})] \leq \hatl(f_\vecw) + \underbrace{\E_{\eps \sim \mathcal{N}(0,\sigma)^n} [\hatl(f_{\vecw+\eps})] -\hatl(f_\vecw) }_{\text{expected sharpness}}+ 4\sqrt{\frac{1}{m} \bigg( \underbrace{\frac{\|\vecw\|_2^2}{2\sigma^2} }_{\text{KL}}+ \ln \frac{2m}{\delta} \bigg) },
\end{equation}
\end{small}
Another interesting approach is to set the variance of the perturbation to each parameter with respect to the magnitude of the parameter. For example if $\sigma_i=\alpha \abs{w_i}+\beta$, then the KL term in the above expression changes to $\sum_i\frac{w_i^2}{2\sigma_i^2} $. 

The above generalization guarantees give a clear way to think about capacity control jointly in terms of both the expected sharpness and the norm, and as we discussed earlier indicates that sharpness by itself cannot control the capacity without considering the scaling. In the above generalization bound, norms and sharpness interact in a direct way depending on $\sigma$, as increasing the norm by decreasing $\sigma$ causes decrease in sharpness and vice versa. It is therefore important to find the right balance between the norm and sharpness by choosing $\sigma$ appropriately in order to get a reasonable bound on the capacity.

In our experiments we observe that looking at both these measures
jointly indeed makes a better predictor for the generalization error.
As discussed earlier, \citet{dziugaite2017computing} numerically
optimize the overall PAC-Bayes generalization bound over a family of
multivariate Gaussian distributions (different choices of
perturbations and priors). Since the precise way the sharpness and
KL-divergence are combined is not tight, certainly not in
\eqref{eq:pacbayes2}, nor in the more refined bound used by
\citet{dziugaite2017computing}, we prefer shying away from numerically
optimizing the balance between sharpness and the KL-divergence.
Instead, we propose using bi-criteria plots, where sharpness and
KL-divergence are plotted against each other, as we vary the
perturbation variance.  For example, in the center and right panels of
Figure \ref{fig:sharpness-true-random} we show such plots for networks
trained on true and random labels respectively.  We see that although
sharpness by itself is not sufficient for explaining generalization in
this setting (as we saw in the left panel), the bi-criteria plots are
significantly lower for the true labels.  Even more so, the change in
the bi-criteria plot as we increase the number of samples is
significantly larger with random labels, correctly capturing the
required increase in capacity.  For example, to get a fixed value of
expected sharpness such as $\error=0.05$, networks trained with random
labels require higher norm compared to those trained with true labels.
This behavior is in agreement with our earlier discussion, that
sharpness is sensitive to scaling of the parameters and is not a
capacity control measure as it can be artificially changed by scaling
the network. However, combined with the norm, sharpness does seem to
provide a capacity measure.

% In our evaluation, instead of combining both the measures, we look at the bi-criteria plots (see Figure~\ref{fig:sharpness-true-random}), where we plot both sharpness and norm for different choices of perturbations and priors.

%Gaussian distributions,
%This connection between sharpness and the PAC-Bayesian framework was
%also recently noticed by \citet{dziugaite2017computing}, who optimize
%the PAC-Bayes generalization bound over a family of multivariate
%Gaussian distributions, extending the work of \citet{langford2001not}.
%They show that the optimized PAC-Bayes bounds are numerically
%non-vacuous for feedforward networks trained on a binary classification
%variant of MNIST dataset.

\section{Empirical Investigation}\label{sec:empirical}

In this section we investigate the ability of the discussed measures
to explain the the generalization phenomenon discussed in the
Introduction.  We already saw in Figures \ref{fig:norm-true-random}
and \ref{fig:sharpness-true-random} that these measures capture the
difference in generalization behavior of models trained on true or
random labels, including the increase in capacity as the sample size
increases, and the difference in this increase between true and random
labels.

\subsection*{Different Global Minima}

Given different global minima of the training loss on the same
training set and with the same model class, can these measures
indicate which model is going to generalize better? In order to verify
this property, we can calculate each measure on several different
global minima and see if lower values of the measure imply lower
generalization error. In order to find different global minima for the
training loss, we design an experiment where we force the optimization
methods to converge to different global minima with varying
generalization abilities by forming a confusion set that includes
samples with random labels. The optimization is done on the loss that
includes examples from both the confusion set and the training set.
Since deep learning models have very high capacity, the optimization
over the union of confusion set and training set generally leads to a
point with zero error over both confusion and training sets which thus
is a global minima for the training set.

We randomly select a subset of CIFAR10 dataset with 10000 data points
as the training set and our goal is to find networks that have zero
error on this set but different generalization abilities on the test
set. In order to do that, we train networks on the union of the
training set with fixed size 10000 and confusion sets with varying
sizes that consists of CIFAR10 samples with random labels; and we
evaluate the learned model on an independent test set. The trained
network achieves zero training error but as shown in Figure
\ref{fig:cifar-core}, the test error of the model increases with
increasing size of the confusion set. The middle panel of this Figure
suggests that the norm of the learned networks can indeed be
predictive of their generalization behavior. However, we again observe
that sharpness has a poor behavior in these experiments. The right
panel of this figure also suggests that PAC-Bayes measure of joint
sharpness and KL divergence, has better behavior - for a fixed
expected sharpness, networks that have higher generalization error,
have higher norms.

%Experiments in \cite{zhang2017understanding} suggest that neural network architectures used in practice can fit a large dataset with random labels. Hence, complexity measure based on number of parameters do not capture the generalization behavior in this case. Here we repeat the same experiments on training with random labels and compute norm based measure for various norms. 

%But if norm is a sensible complexity measure for the networks, it is expected to generally observe higher norm for networks trained with random labels compare to those trained with true labels. In order to understand this relationship better, we go back to our understanding of margin for linear separators. If the training data is separable with large margin (low norm), we expect the linear separator to maximize the margin given enough data points. 
\subsection*{Increasing Network Size}

We also repeat the experiments conducted by \citet{neyshabur15b} where
a fully connected feedforward network is trained on MNIST dataset with
varying number of hidden units and we check the values of different
complexity measures on each of the learned networks.The left panel in
Figure \ref{fig:hidden} shows the training and test error for this
experiment. While 32 hidden units are enough to fit the training data,
we observe that networks with more hidden units generalize better.
Since the optimization is done without any explicit regularization,
the only possible explanation for this phenomenon is the implicit
regularization by the optimization algorithm. Therefore, we expect a
sensible complexity measure to decrease beyond 32 hidden units and
behave similar to the test error. Different measures are reported for
learned networks. The middle panel suggest that all margin/norm based
complexity measures decrease for larger networks up to 128 hidden
units. For networks with more hidden units, $\ell_2$ norm and
$\ell_1$-path norm increase with the size of the network. The middle
panel suggest that $\ell_2$-path norm can provide some explanation for
this phenomenon. However, as we discussed in
Section~\ref{sec:summary}, the actual complexity measure based on
$\ell_2$-path norm also depends on the number of hidden units and
taking this into account indicates that the measure based on
$\ell_2$-path norm cannot explain this phenomenon. This is also the
case for the margin based measure that depends on the spectral norm.
In subsection \ref{subsec:lipschitz} we discussed another complexity
measure that also depends the spectral norm through Lipschitz
continuity or robustness argument. Even though this bound is very
loose, it is monotonic with respect to the spectral norm that is
reported in the plots. Unfortunately, we do observe some increase in
spectral norm by increasing number of hidden units beyond 512. The
right panel shows that the joint PAC-Bayes measure decrease for larger
networks up to size 128 but fails to explain this generalization
behavior for larger networks. This suggests that the measures looked so
far are not sufficient to explain all the generalization phenomenon
observed in neural networks.

\begin{figure}[t]
\centering
\includegraphics[width=.32\textwidth]{\figdir/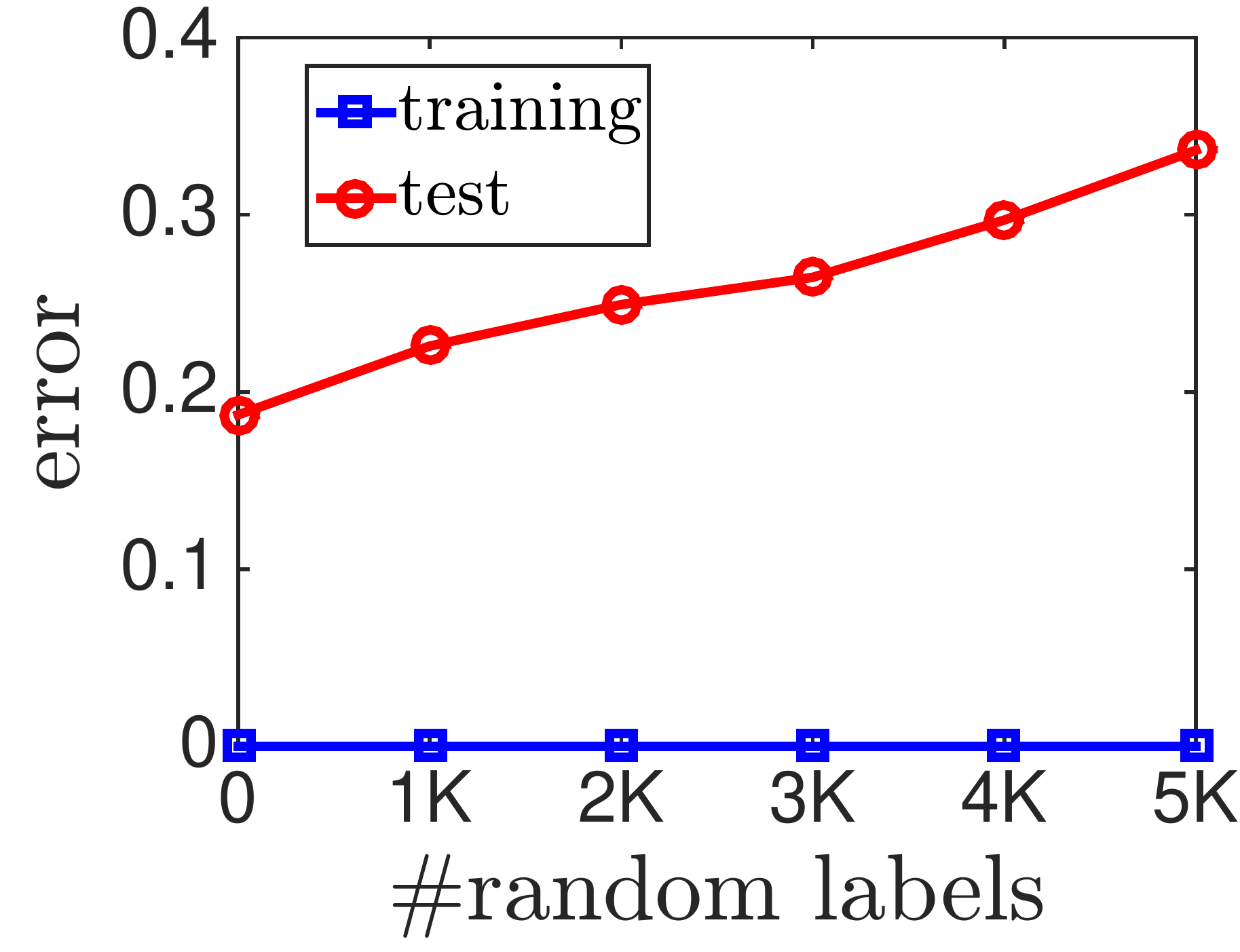}
\includegraphics[width=.32\textwidth]{\figdir/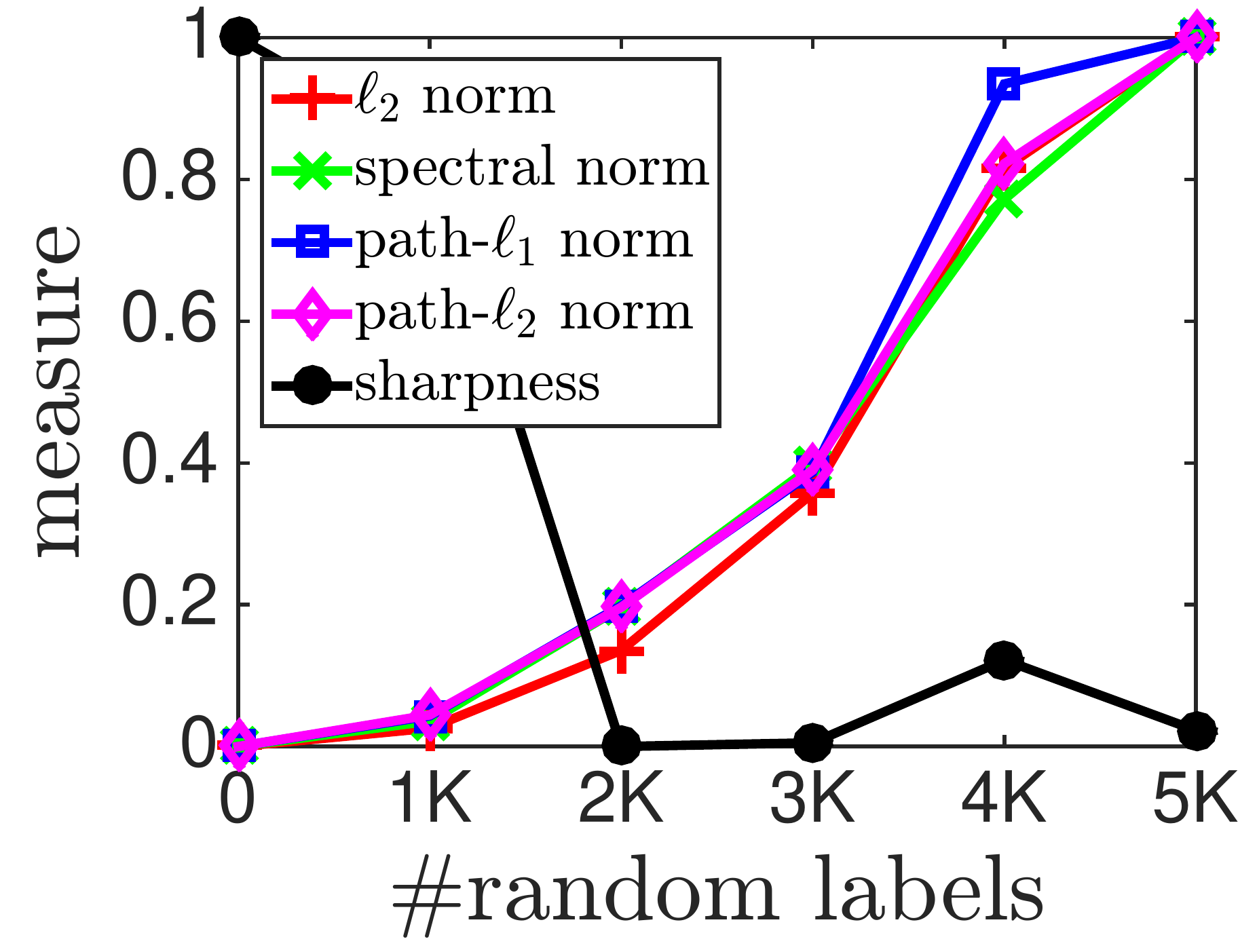}
\includegraphics[width=.32\textwidth]{\figdir/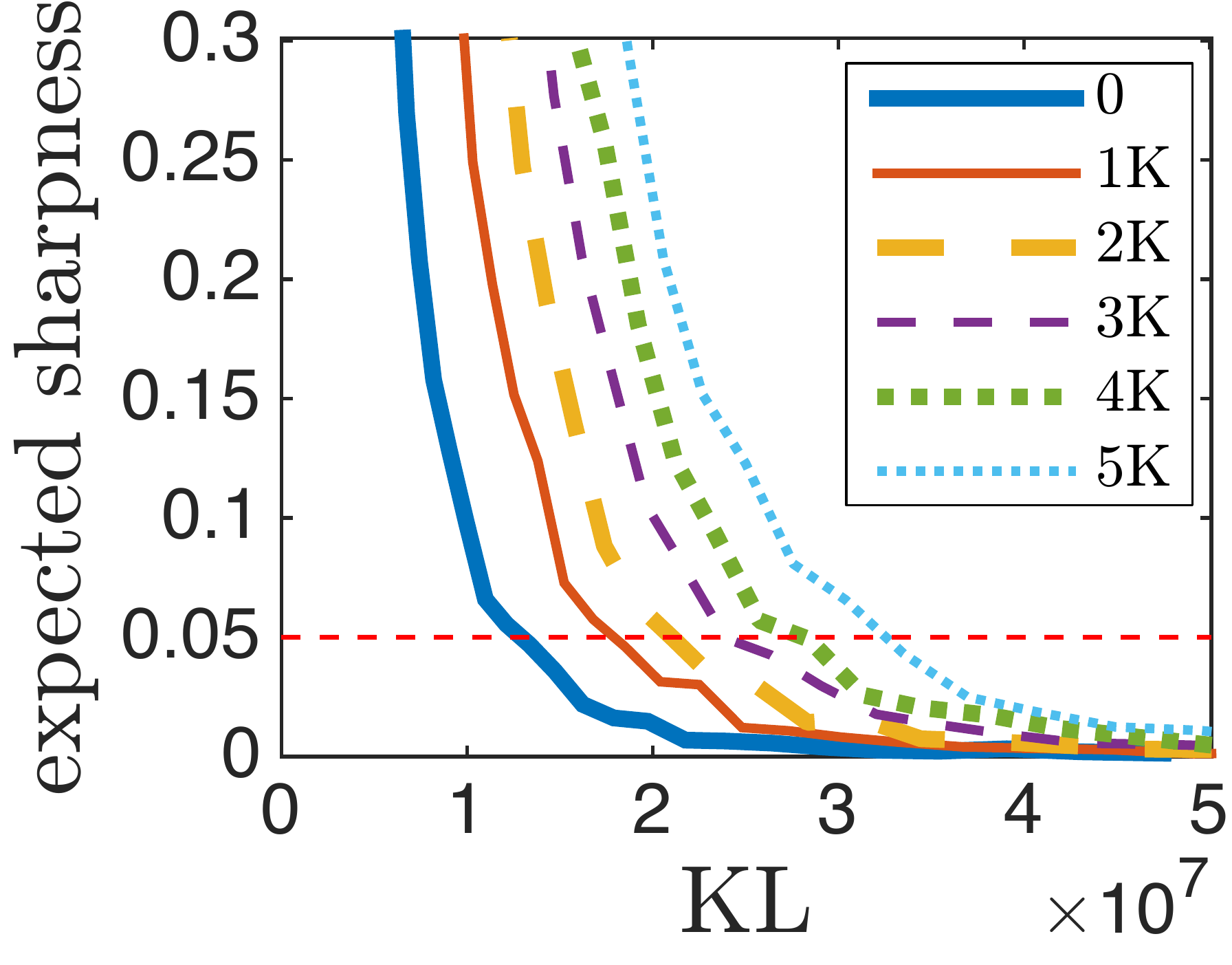}
\caption{\small Experiments on global minima with poor generalization. For each experiment, a VGG network is trained on union of a subset of CIFAR10 dataset with size 10000 containing samples with true labels and another subset of CIFAR10 datasets with varying size containing random labels. The learned networks are all global minima for the objective function on the subset with true labels. The left plot indicates the training and test errors based on the size of the set with random labels. The plot in the middle shows change in different measures based on the size of the set with random labels. The plot on the right indicates the relationship between expected sharpness and KL in PAC-bayes for each of the experiments. Measures are calculated as explained in Figures \ref{fig:norm-true-random} and \ref{fig:sharpness-true-random}.}
\label{fig:cifar-core}
\end{figure}

\begin{figure}[t]
\centering
\includegraphics[width=.32\textwidth]{\figdir/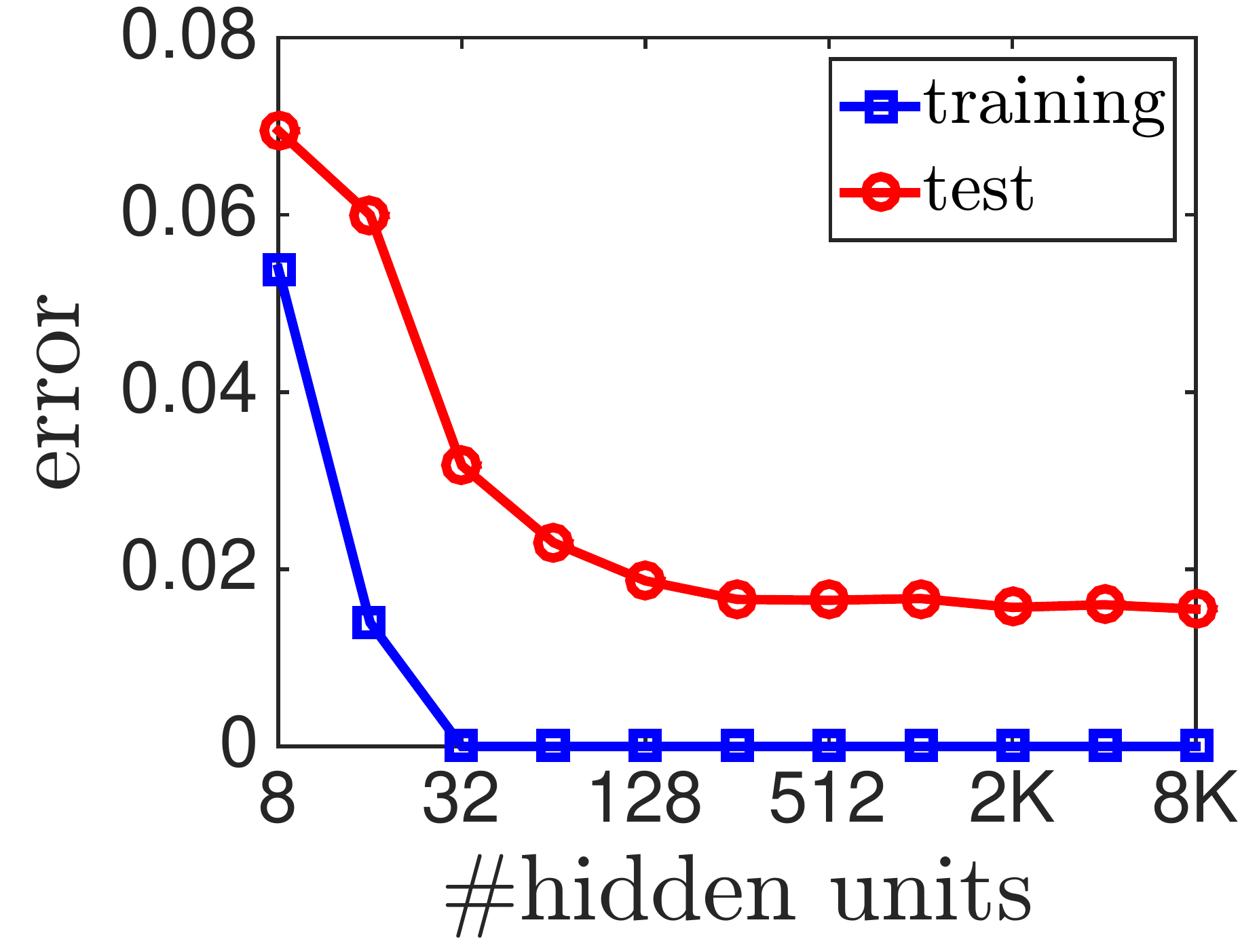}
\includegraphics[width=.32\textwidth]{\figdir/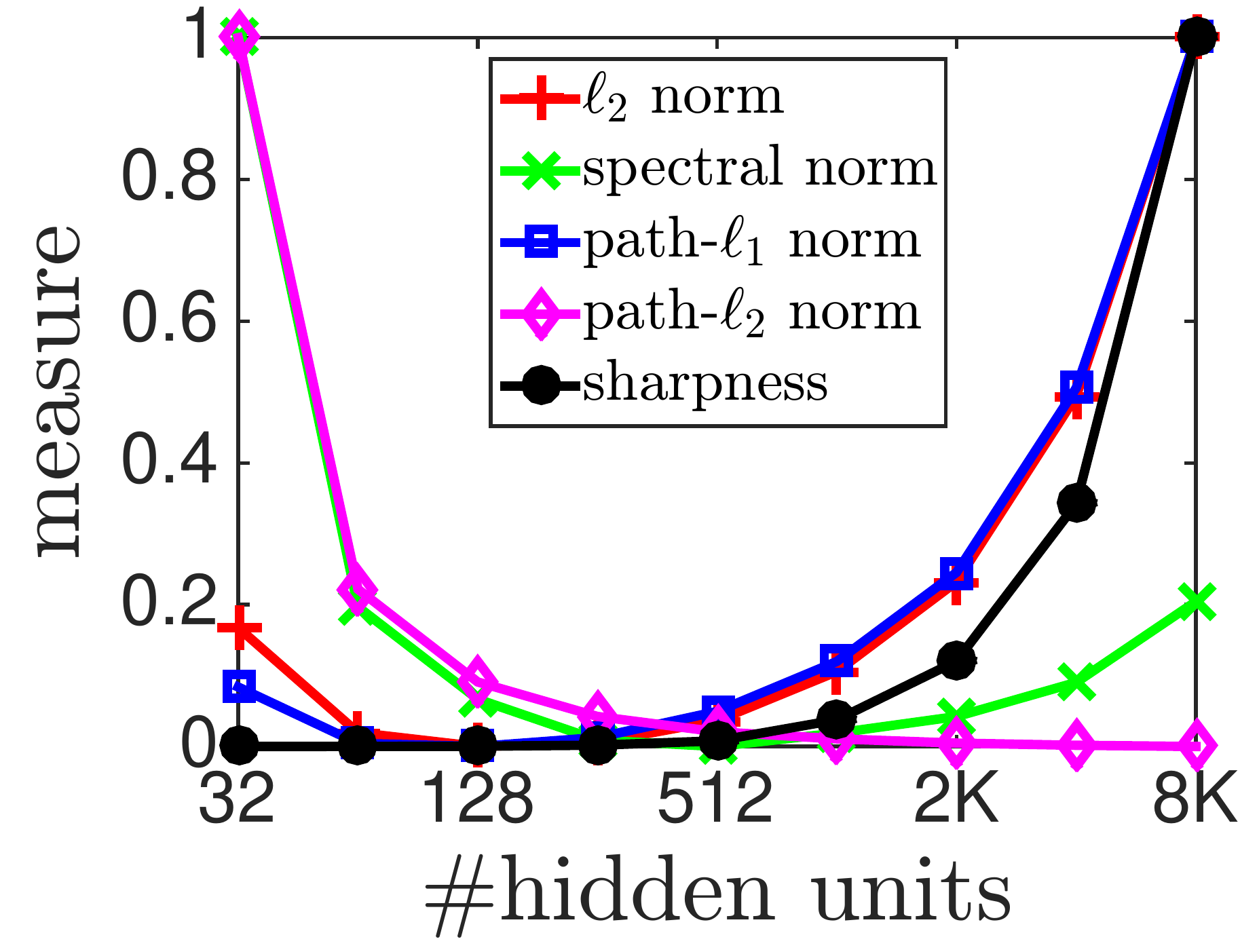}
\includegraphics[width=.32\textwidth]{\figdir/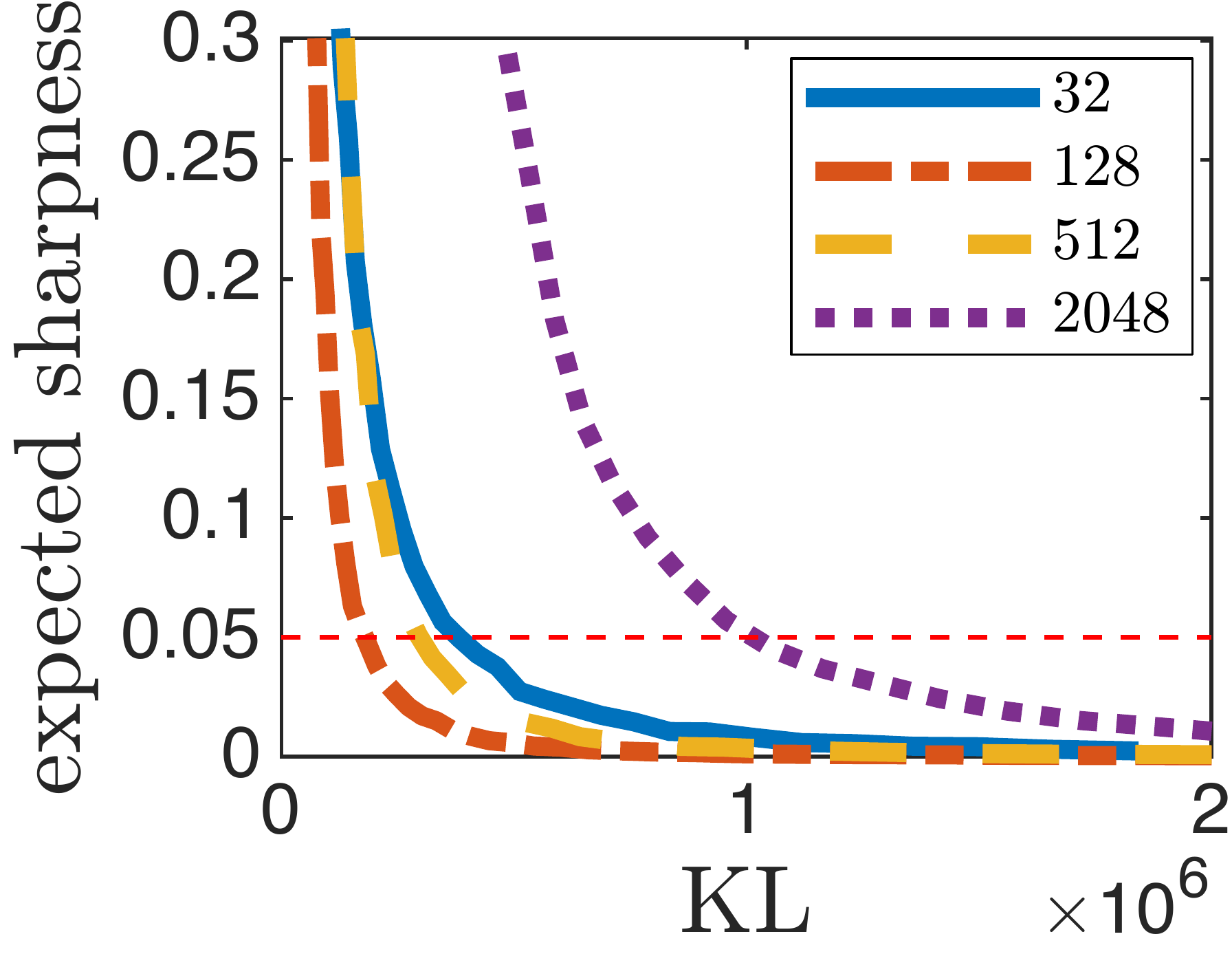}
\caption{\small The generalization of two layer perceptron trained on MNIST dataset with varying number of hidden units. The left plot indicates the training and test errors. The test error decreases as the size increases. The middle plot shows different measures for each of the trained networks. The plot on the right indicates the relationship between expected sharpness and KL in PAC-Bayes for each of the experiments. Measures are calculated as explained in Figures \ref{fig:norm-true-random} and \ref{fig:sharpness-true-random}.}
\label{fig:hidden}
\end{figure}

\section{Bounding Sharpness}\label{sec:pac_bayes}

So far we have discussed margin based and sharpness based complexity measures to understand capacity. We have also discussed how sharpness based complexity measures in combination with norms characterize the generalization behavior under the PAC-Bayes framework. In this section we study the question of what affects the sharpness of neural networks? For the case of linear predictors, sharpness only depends on the norm of the predictor. In contrast, for multilayered networks, interaction between the layers plays a major role and consequently two different networks with the same norm can have drastically different sharpness values. For example, consider a network where some subset of the layers despite having non-zero norm interact weakly with their neighbors, or are almost orthogonal to each other. Such a network will have very high sharpness value compared to a network where the neighboring layers interact strongly.

In this section we establish sufficient conditions to bound the expected sharpness of a feedforward network with ReLU activations.  Such conditions serve as a useful guideline in studying what helps an optimization method to converge to less sharp optima. Unlike existing generalization bounds~\cite{bartlett2002rademacher,NeyTomSre15,luxburg2004distance,xu2012robustness,sokolic2016generalization}, our sharpness based bound does not suffer from exponential dependence on depth.

Now we discuss the conditions that affect the sharpness of a network. As discussed earlier, weak interactions between layers can cause the network to have high sharpness value. Condition $C1$ below prevents such weak interactions (cancellations). A network can also have high sharpness if the changes in the number of activations is exponential in the perturbations to its weights, even for small perturbations. Condition $C2$ avoids such extreme situations on activations. Finally, if a non-active node with large weights becomes active because of the perturbations in lower layers, that can lead to huge changes to the output of the network. Condition $C3$ prevents having such spiky (in magnitude) hidden units. This leads us to the following three conditions, that help in avoiding such pathological cases.

\begin{itemize}
\item[$(C1):$] Given $x$, let $x=W_0$ and $D_0 =I$. Then, for all $0 \leq a < c < b \leq d, \| \left(\Pi_{i=a}^{b} D_{i}W_i \right)\|_F \geq  \frac{\mu}{\sqrt{h_c}}  \| \Pi_{i=c+1}^{b} D_{i}W_i  \|_F \|  \left(\Pi_{i=a}^{c} D_{i}W_i \right)\|_F $.
\item[$(C2):$] Given $x$, for any level $k$, $\frac{1}{h_k} \sum_{i \in [h_k]} 1_{W_{k,i} \Pi_{j=1}^{k-1} D_j W_j x \leq \delta} \leq C_2 \delta$.
\item[$(C3):$] For all $i$, $\| W_{i}\|_{2,\infty}^2  h_i \leq C_3^2 \|D_i W_i\|_F^2$.
\end{itemize}

Here, $W_{k, i}$ denotes the weights of the $i^{th}$ output node in layer $k$. $\| W_{i}\|_{2,\infty}$ denotes the maximum $L2$ norm of a hidden unit in layer $i$. Now we state our result on the generalization error of a ReLU network, in terms of average sharpness and its norm. Let $\|x\| = 1$ and $h=\max_{i=1}^d h_i$. 

\begin{thm}\label{thm:relu}
Let $\eps_i$ be a random $h_i \times h_{i-1}$ matrix with each entry distributed according to $\mathcal{N}(0,\sigma_i^2)$. Then, under the conditions $C1, C2, C3$,  with probability $\geq 1-\delta$, 
\begin{small}
\begin{align*}
&\E_{\eps \sim \mathcal{N}(0,\sigma)^n} [L(f_{\vecw+\eps})] - \hatl(f_\vecw) \leq O \left( \left[\Pi_{i=1}^d \left(1+\gamma_i \right) -1 \right. \right. \\ & \left. \left. + \Pi_{i=1}^d \left( 1+\gamma_i C_2 C_3 \right)\left(\Pi_{i=1}^d (1+ \gamma_i  C_{\delta} C_2)- 1\right) \right] C_{L} \sum_x \frac{\|f_\vecw(x)\|_F}{m} \right)+ \sqrt{\frac{1}{m}\left(\sum_{i=1}^d \frac{\|W_i\|_F^2 }{\sigma_i^2 } +  \ln \frac{2m}{\delta} \right)}.
\end{align*}
\end{small}
where $\gamma_i = \frac{\sigma_i \sqrt{h_i} \sqrt{h_{i-1}}}{\mu^2 \|W_i\|_F}$ and $C_{\delta}=2\sqrt{\ln(dh/\delta)} $.
\end{thm}

\removed{
\begin{lem}\label{lem:sharpness}
Let $\eps_i$ be a random matrix with same dimensions as $W_i$ and each entry distributed according to $\mathcal{N}(0,\sigma_i^2)$ and let $h=\max_{i=1^d}h_i$ and $C_\delta=2\sqrt{\ln(dh/\delta)}$. Then under the conditions $C1, C2, C3$,  if $\sigma_i \leq \frac{d\mu^2 \norm{D_iW_i}_F}{h(1 + C_2(C_3+C_\delta))}$, then with probability $\geq 1-\delta$, 
\begin{equation}
\E_{\eps \sim \mathcal{N}(0,\sigma)^n} [\hatl(f_{\vecw+\eps})] - \hatl(f_\vecw) \leq O\left(\frac{ C_L(1 + C_2(C_3+C_\delta)\sum_{i=1}^d \sigma_i}{m}\sum_{x\in S}\norm{f_\vecw(x)}\right)
\end{equation}
\end{lem}
}

To understand the above generalization error bound,  consider choosing $\gamma_i =\frac{\sigma}{C_{\delta}d}$, and we get a bound that simplifies as follows:
\begin{align*}
 \E_{\eps \sim \mathcal{N}(0,\sigma)^n} [L(f_{\vecw+\eps})] - \hatl(f_\vecw) & \leq O \left(  \sigma \left( 1 + (1+ \sigma C_2 C_3)  C_2 \right) C_L \frac{\sum_x\|f_\vecw(x)\|_F}{m} \right)  \\ & \quad \quad+  \sqrt{\frac{1}{m}\left(\frac{d^2}{ \mu^4}\sum_{i=1}^d  \frac{h_i h_{i-1} }{\sigma^2}+  \ln \frac{2m}{\delta} \right)}
\end{align*}

If we choose large $\sigma$, then the network will have higher expected sharpness but smaller 'norm' and vice versa. Now one can optimize over the choice of $\sigma$ to balance between the terms on the right hand side and get a better capacity bound. For any reasonable choice of $\sigma$, the generalization error above, depends only linearly on depth and does not have any exponential dependence, unlike other notions of generalization. Also the error gets worse with decreasing $\mu$ and increasing $C_2, C_3$ as the sharpness of the network increases which is in accordance with our discussion of the conditions above.

\begin{figure}[t]
\centering
\includegraphics[width=.32\textwidth]{\figdir/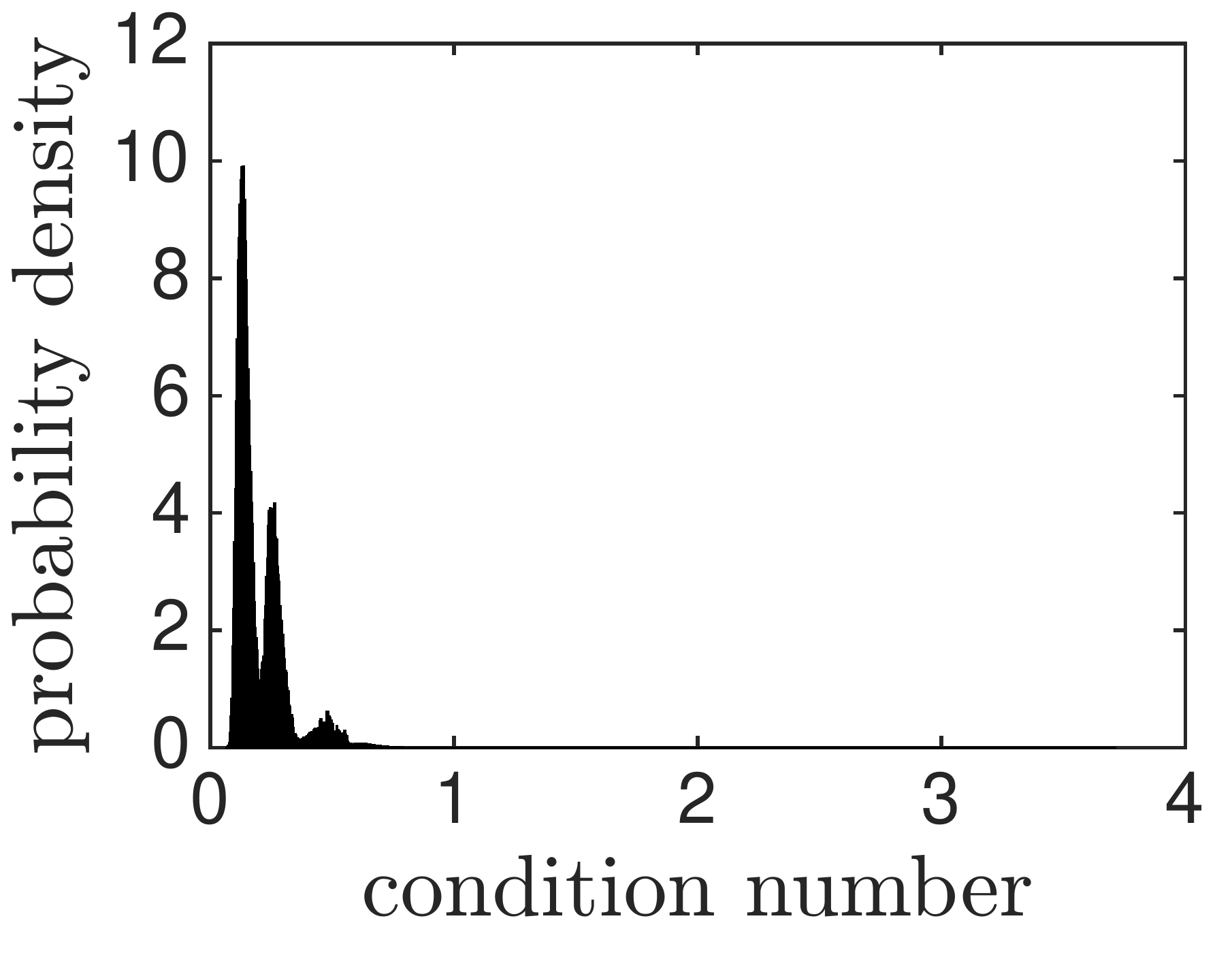}
\includegraphics[width=.32\textwidth]{\figdir/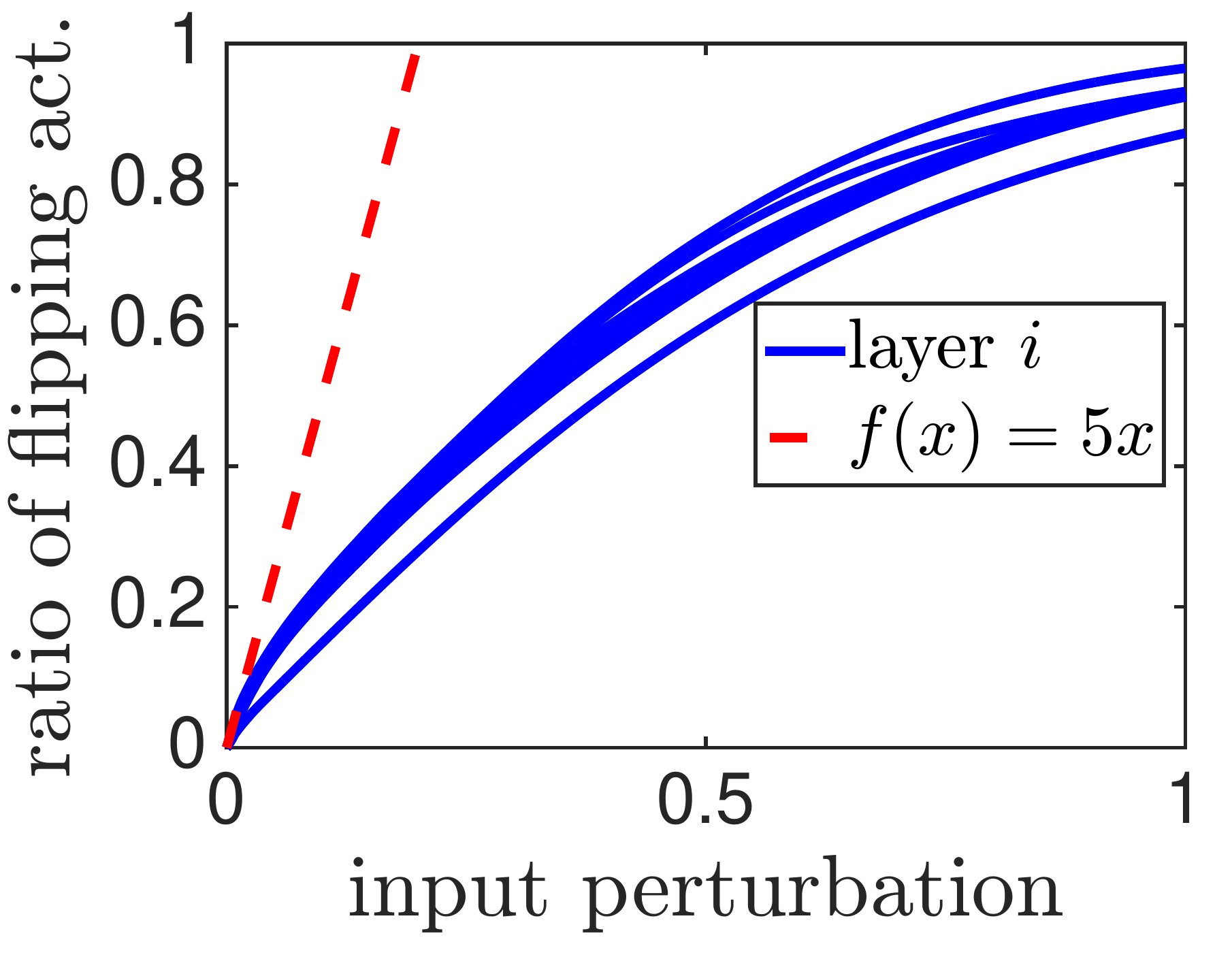}
\includegraphics[width=.32\textwidth]{\figdir/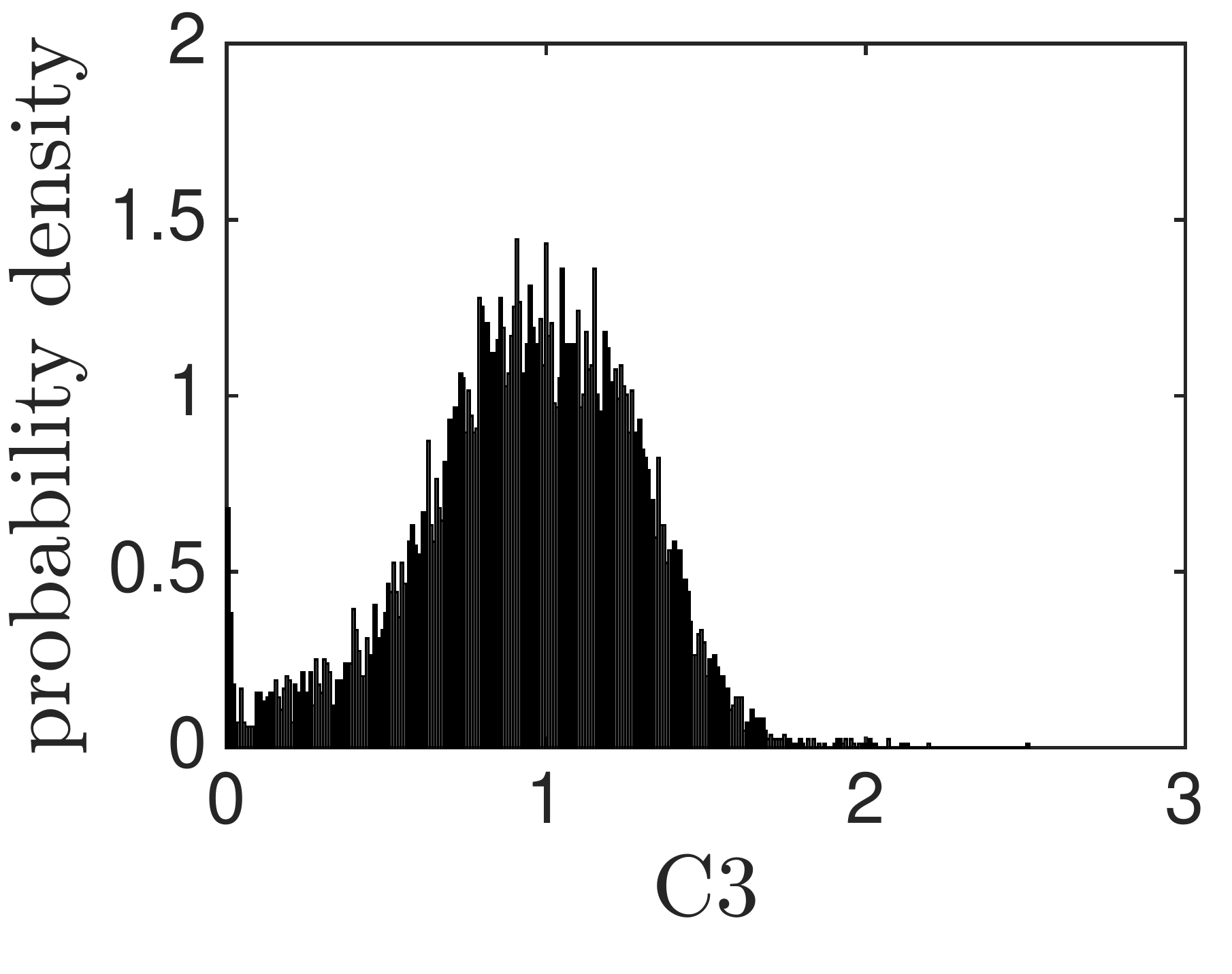}\\
\caption{\small Verifying the conditions of Theorem \ref{thm:relu} on a 10 layer perceptron with 1000 hidden units in each layer, i.e. more than 10,000,000 parameters on MNIST. We have numerically checked that all values are within the displayed range. \textbf{Left}: $C1$: condition number of the network, i.e. $\frac{1}{\mu}$. \textbf{Middle}:  $C2$: the ratio of activations that flip based on magnitude of perturbation. \textbf{Right}:  $C3:$ the ratio of norm of incoming weights to each hidden units with respect to average of the same quantity over hidden units in the layer.}
\label{fig:conditions_verify1}
\end{figure}

Additionally the conditions $C1-C3$ actually hold for networks trained in practice as we verify in Figure~\ref{fig:conditions_verify1}, and our experiments suggest that, $\mu \geq 1/4, C2 \leq 5$ and $C3 \leq 3$. More details on the verification and comparing the conditions on learned network with those of random weights, are presented in the appendix.

\paragraph{Proof of Theorem~\ref{thm:relu}}
We bound the expectation as follows:
\begin{align}
&\E \abs{ \hatl(f_{\vecw+\eps}(x)) -\hatl(f_{\vecw}(x))} \nonumber \\ 
&\quad \quad  \leq C_{L} \E \| f_{\vecw+\eps}(x) -f_\vecw(x)\|_F \nonumber \\
&\quad \quad \stackrel{(i)}{=}  C_{L}\E \| (W+\eps)_d \left(\Pi_{i=1}^{d-1} \hatD_{i} (W+\eps)_i \right)*x - W_d \left(\Pi_{i=1}^{d-1} D_{i}W_i \right)*x \|_F  \nonumber \\
& \quad \quad \leq  C_{L} \E \| (W+\eps)_d \left(\Pi_{i=1}^{d-1} D_{i} (W+\eps)_i \right)*x - W_d \left(\Pi_{i=1}^{d-1} D_{i}W_i \right)*x \|_F \nonumber \\ & \quad \quad \quad \quad + C_{L} \E \| (W+\eps)_d \left(\Pi_{i=1}^{d-1} \hatD_{i} (W+\eps)_i \right)*x - (W+\eps)_d \left(\Pi_{i=1}^{d-1} D_{i} (W+\eps)_i \right)*x \|_F \nonumber \\
&\quad \quad \leq  C_{L} \E \| (W+\eps)_d \left(\Pi_{i=1}^{d-1} D_{i} (W+\eps)_i \right)*x - W_d \left(\Pi_{i=1}^{d-1} D_{i}W_i \right)*x \|_F +C_{L} \E \|\err_d\|_F,  \label{eq:thm_relu1}
\end{align}
where $\err_d =  \| (W+\eps)_d \left(\Pi_{i=1}^{d-1} \hatD_{i} (W+\eps)_i \right)*x - (W+\eps)_d \left(\Pi_{i=1}^{d-1} D_{i} (W+\eps)_i \right)*x \|_F$. $(i)$ $\hatD_i$ is the diagonal matrix with 0's and 1's corresponding to the activation pattern of the perturbed network $f_{\vecw+\eps}(x)$.

The first term in the equation~\eqref{eq:thm_relu1} corresponds to error due to perturbation of a network with unchanged activations (linear network). Intuitively this is small when any subset of successive layers of the network do no interact weakly with each other (not orthogonal to each other). Condition $C1$ captures this intuition and we bound this error in Lemma~\ref{eq:lem_linear1}. 

\begin{lem}\label{lem:linear}
Let $\eps_i$ be a random $h_i \times h_{i-1}$ matrix with each entry distributed according to $\mathcal{N}(0,\sigma_i^2)$. Then, under the condition $C1$,
\begin{multline*}
\E \| (W+\eps)_d \left(\Pi_{i=1}^{d-1} D_{i}(W+\eps)_i \right)*x - W_d \left(\Pi_{i=1}^{d-1} D_{i}W_i \right)*x\|_F \\ \leq  \left(\Pi_{i=1}^d \left(1+ \frac{\sigma_i \sqrt{h_ih_{i-1}}}{\mu^2 \|D_i W_i\|_F} \right) -1 \right)   \|f_\vecw(x)\|_F. 
\end{multline*}
\end{lem}

The second term in the equation~\eqref{eq:thm_relu1} captures the perturbation error due to change in activations. If a tiny perturbation can cause exponentially many changes in number of active nodes, then that network will have huge sharpness. Condition $C2$ and $C3$ essentially characterize the behavior of sensitivity of activation patterns to perturbations, leading to a bound on this term in Lemma~\ref{lem:recursion}.

\begin{lem} \label{lem:recursion}
Let $\eps_i$ be a random $h_i \times h_{i-1}$ matrix with each entry distributed according to $\mathcal{N}(0,\sigma_i^2)$. Then, under the conditions $C1$, $C2$ and $C3$, with probability $\geq 1-\delta$, for all $ 1 \leq k \leq d$, 
$$\|\hatD_{k} -D_{k}\|_1 \leq  O \left( C_2 h_k  C_{\delta} \sigma_k  \|f^{k-1}_{\vecw}\|_F \right) $$ and 
$$\E\| \err_k\|_F \leq  O \left( \Pi_{i=1}^k \left( 1+\gamma_iC_2 C_3\right)\left(\Pi_{i=1}^k (1+ \gamma_i C_{\delta} C_2) - 1\right)\|f^k_\vecw\|_F \right).$$
where $\gamma_i = \frac{\sigma_i \sqrt{h_i} \sqrt{h_{i-1}}}{\mu^2 \|D_iW_i\|_F}$ and $C_\delta=2\sqrt{\ln(dh/\delta)}$.
\end{lem}

Hence, from Lemma~\ref{eq:lem_linear1} and Lemma~\ref{lem:recursion} we get, 
\begin{small}
\begin{align*}
&\E \abs{ \hatl(f_{\vecw+\eps}(x)) -\hatl(f_{\vecw}(x))} \nonumber \\ 
&\leq \left[\Pi_{i=1}^d \left(1+ \gamma_i \right) -1  + \Pi_{i=1}^d \left( 1+\gamma_i C_2 C_3 \right)\left(\Pi_{i=1}^d (1+ \gamma_i C_{\delta} C_2  )- 1\right) \right] C_{L} \|f_\vecw(x)\|_F.
\end{align*}

Here $\gamma_i =\frac{\sigma_i\sqrt{h_i}\sqrt{h_{i-1}} }{\mu^2 \|D_iW_i\|_F } $. Substituting the above bound on expected sharpness in the PAC-Bayes result (equation~\eqref{eq:pacbayes}), gives the result.
\end{small}
 
\section{Conclusion}
Learning with deep neural networks displays good generalization
behavior in practice, a phenomenon that remains largely unexplained.
In this paper we discussed different candidate complexity measures
that might explain generalization in neural networks.  We outline a
concrete methodology for investigating such measures, and report on
experiments studying how well the measures explain different
phenomena.  While there is no clear choice yet, some combination of
expected sharpness and norms do seem to capture much of the
generalization behavior of neural networks.  A major issue still left
unresolved is how the choice of optimization algorithm biases such
complexity to be low, and what is the precise relationship between
optimization and implicit regularization.

\bibliographystyle{abbrvnat}
\bibliography{ref}

%%%%%%%%%%%%%%%%%%%%%%%%%%%%%%%%%%%%%%%%%%%%%%%%%%%%%%%%%%%%%%%%%%%%%%%%%
%\clearpage
\appendix

%%%%%%%%%%%%%%%%%%%%%%%%%%%

\section{Experiments Settings}
In experiment with different network sizes, we train a two layer perceptron with ReLU activation and varying number of hidden units without Batch Normalization or dropout. In the rest of the experiments, we train a modified version of the VGG architecture \cite{simonyan2014very} with the configuration $2\times [64,3,3,1]$, $2\times [128,3,3,1]$, $2\times [256,3,3,1]$, $2\times [512,3,3,1]$ where we add Batch Normalization before ReLU activations and apply $2\times 2$ max-pooling with window size 2 and dropout after each stack. Convolutional layers are followed by $4\times 4$ average pooling, a fully connected layer with 512 hidden units and finally a linear layer is added for prediction. 

In all experiments we train the networks using stochastic gradient descent (SGD) with mini-batch size 64, fixed learning rate 0.01 and momentum 0.9 without weight decay. In all experiments where achieving zero training error is possible, we continue training until the cross-entropy loss is less than $10^{-4}$.

When calculating norms on a network with a Batch Normalization layer, we reparametrize the network to one that represents the exact same function without Batch Normalization as suggested in \cite{neyshabur16}. In all our figures we plot norm divided by margin to avoid scaling issues (see Section~\ref{sec:summary}), where we set the margin over training set $S$ to be $5^{th}$-percentile of the margins of the data points in $S$, i.e. $\text{Prc}_5\left\{f_\vecw(x_i)[y_i] - \max_{y\neq y_i} f_\vecw(x)[y] | (x_i,y_i)\in S\right\}.$ We have also investigated other versions of the margin and observed similar behavior to this notion. 

We calculate the sharpness, as suggested in \cite{keskar2016large} - for each parameter $w_i$ we bound the magnitude of perturbation by $\alpha(\abs{w_i}+1)$ for $\alpha=5.10^{-4}$. In order to compute the maximum perturbation (maximize the loss), we perform 2000 updates of stochastic gradient ascent starting from the minimum, with mini-batch size 64, fixed step size 0.01 and momentum 0.9.

To compute the expected sharpness, we perturb each parameter $w_i$ of the model with noise generated from Gaussian distribution with zero mean and standard deviation, $\alpha(10\abs{w_i}+1)$. The expected sharpness is average over 1000 random perturbations each of which are averaged over a mini-batch of size 64. We compute the expected sharpness for different choices of $\alpha$. For each value of $\alpha$ the KL divergence can be calculated as $\frac{1}{\alpha^2}\sum_i \left(\frac{w_i}{(10\abs{w_i}+1)}\right)^2$.

\section{Proofs}

\subsection{Proof of Lemma~\ref{lem:linear}}
\begin{proof}
Define $g_{\vecw,\eps,s}(x)$ as the network $f_\vecw$ with weight $W_i$ in every layer $i\in s$ replaced by $\eps_i$. Hence,
\begin{align}
 \| (W+\eps)_d &\left(\Pi_{i=1}^{d-1} D_{i}(W+\eps)_i \right)*x - W_d \left(\Pi_{i=1}^{d-1} D_{i}W_i \right)*x\|_F \nonumber \\
& \leq  \|\sum_i g_{\vecw,\eps,\{i\}}(x)\|_F +\|\sum_{i,j} g_{\vecw,\eps,\{i,j\}}(x)\|_F + \cdots + \| f_{\eps}(x)\|_F  \label{eq:lem_linear1}
\end{align}

\noindent {\bf Base case:} First we show the bound for terms with one noisy layer. Let $g_{\vecw,\eps,\{k\}}(x)$ denote $f_\vecw(x)$ with weights in layer $k$, $W_k$ replaced by $\eps_k$. Now notice that,
\begin{align*}
\E \|g_{\vecw,\eps,\{k\}}(x)\|_F &= \E \|W_d \Pi_{i=k+1}^{d-1} D_{i}W_i* D_k \eps_k * \left(\Pi_{i=1}^{k-1} D_{i}W_i \right)*x \|_F \\
&\stackrel{(i)}{\leq}  \sigma_k   \|W_d \Pi_{i=k+1}^{d-1} D_{i}W_i\|_F  \|  \|\left(\Pi_{i=1}^{k-1} D_{i}W_i \right)*x \|_F \\
&\stackrel{(ii)}{\leq}  \sigma_k \frac{\sqrt{h_k h_{k-1}}}{\mu^2 \|D_k W_k\|_F} \|W_d \left(\Pi_{i=1}^{d-1} D_{i}W_i \right)*x \|_F \\
&=\sigma_k \frac{\sqrt{h_k h_{k-1} }}{\mu^2 \|D_k W_k\|_F} \| f_\vecw(x)\|_F.
\end{align*}
$(i)$ follows from Lemma~\ref{lem:gauss_product}. $(ii)$ follows from condition $C1$.

\noindent {\bf Induction step:} Let for any set $s \subset [d], |s| =k$, the following holds:   $$\E \|g_{\vecw,\eps,s}(x)\|_F \leq  \|f_\vecw(x)\|_F \Pi_{i \in s} \sigma_i \frac{\sqrt{h_i h_{i-1}}}{\mu^2\| D_i W_i\|_F} . $$

We will prove this now for terms with $k+1$ noisy layers. 
\begin{align*}
\E \|g_{\vecw,\eps,s \cup \{j\}},x)\|_F &\leq  \sigma_j \frac{\sqrt{h_j h_{j-1}}}{\mu^2 \|D_j W_j\|} \E \|g_{\vecw,\eps,s}(x)\|_F \\
&\leq  \sigma_j \frac{\sqrt{h_j h_{j-1}}}{\mu^2\|D_j W_j\|}    \|f_\vecw(x)\|_F \Pi_{i \in s} \sigma_i \frac{\sqrt{h_i h_{i-1}}}{\mu^2 \| D_i W_i\|_F} \\
&= \|f_\vecw(x)\|_F \Pi_{i \in s \cup \{j\}} \sigma_i \frac{\sqrt{h_i h_{i-1}}}{\mu^2 \| D_i W_i\|_F}
\end{align*}

Substituting the above expression in equation~\eqref{eq:lem_linear1} gives,
 \begin{multline*}
 \| (W+\eps)_d \left(\Pi_{i=1}^{d-1} D_{i}(W+\eps)_i \right)*x - W_d \left(\Pi_{i=1}^{d-1} D_{i}W_i \right)*x\|_F  \\ \leq  \left(\Pi_{i=1}^d \left(1+ \frac{\sigma_i \sqrt{h_ih_{i-1}}}{\mu^2 \|D_i W_i\|_F} \right) -1 \right)  \|f_\vecw(x)\|_F.
\end{multline*}

\end{proof}

\subsection{Proof of Lemma~\ref{lem:recursion}}
\begin{proof}
We prove this lemma by induction on $k$. Recall that $\hatD_i$ is the diagonal matrix with 0's and 1's corresponding to the activation pattern of the perturbed network $f_{\vecw+\eps}(x)$. Let $C_\delta=2\sqrt{\ln(dh/\delta)}$ and  $1_{E}$ denote the indicator function, that is $1$ if the event $E$ is true, $0$ else. We also use $f_\vecw^k(x)$ to denote the network truncated to level $k$, in particular $f_\vecw^k(x) =\Pi_{i=1}^k D_k W_k x$.

\noindent {\bf Base case:}

\begin{multline*}
\|\hatD_1 - D_1\|_1 = \sum_i 1_{\inner{(W+\eps)_{1,i}}{x}*\inner{W_{1,i}}{x} <0} =\sum_i 1_{\inner{(\vecW)_{1,i}}{x}^2 < -\inner{(\eps)_{1,i}}{x}*\inner{(\vecW)_{1,i}}{x}} \\ \leq \sum_i 1_{\abs{\inner{(\vecW)_{1,i}}{x}} < \abs{\inner{(\eps)_{1,i}}{x}}}.
\end{multline*}
Since $\eps_1$ is a random Gaussian matrix, and $\|x\| \leq 1$, for any $i$, $\abs{\inner{(\eps)_{1,i}}{x}} \leq 2\sigma_1\sqrt{\ln(dh/\delta)}=\sigma_1C_\delta$ with probability greater than $1-\frac{\delta}{d}$. Hence, with probability $\geq 1-\frac{\delta}{d}$,

\begin{align*}
\|\hatD_1 - D_1\|_1 \leq \sum_i 1_{\abs{\inner{(\vecW)_{1,i}}{x}} \leq \sigma_1C_\delta}  \leq C_2 h_1\sigma_1C_\delta.
\end{align*}

This completes the base case for $k=1$. $\hatD_1$ is a random variable that depends on $\eps_1$. Hence, in the remainder of the proof, to avoid this dependence, we separately bound $\hatD_1 -D$ using the expression above and compute expectation only with respect to $\eps_1$. With probability $\geq 1-\frac{\delta}{d}$,

\begin{align*}
\E \| \err_1 \|_F &= \E \|\hatD_1*(W+\eps)_1 x - D_1*(W+\eps)_1 x \|_F \\
&\leq \E \|(\hatD_1- D_1)*W_1 x \|_F + \E \|(\hatD_1- D_1)*\eps_1 x \|_F \\
& \stackrel{(i)}{\leq} \sqrt{C_2 h_1\sigma_1C_\delta} \sigma_1 + \sqrt{C_2 h_1\sigma_1C_\delta} \sigma_1\\
& = 2 \sqrt{C_2 h_1\sigma_1C_\delta} \sigma_1.
\end{align*}
$(i)$ follows because, each hidden node in  $\E \|(\hatD_1- D_1)*W_1 x \|_F$ has norm less than $\sigma_1 C_{\delta}$ (as it changed its activation), number of such units is less than $C_2 h_1 \sigma_1 C_{\delta}$.

$k=1$ case does not capture all the intricacies and dependencies of higher layer networks. Hence we also evaluate the bounds for $k=2$.

\begin{align*}
\|\hatD_2 - D_2\|_1 \leq \sum_i 1_{\inner{(W+\eps)_{2,i}}{f^1_{\vecw+\eps}}*\inner{W_{2,i}}{f^1_\vecw} \leq 0 } \leq \sum_i 1_{\abs{\inner{W_{2,i}}{f^1_{\vecw}}} \leq \abs{\inner{\eps_{2,i}}{f^1_{\vecw+\eps}}} + \abs{\inner{W_{2,i}}{f^1_{\vecw+\eps}-f^1_{\vecw}}} }
\end{align*}

Now, with probability $\geq 1-\frac{2\delta}{d}$ we get:

\begin{align*}
&\abs{\inner{\eps_{2,i}}{f^1_{\vecw+\eps}}} + \abs{\inner{W_{2,i}}{f^1_{\vecw+\eps}-f^1_{\vecw}}} \\ &\leq C_{\delta} \sigma_2 \left( \|f^1_\vecw\|_F+ 2 \sqrt{C_2 h_1 \sigma_1 C_{\delta}} \sigma_1 \right) + \|W_{2, i}\| 2 \sqrt{C_2 h_1 \sigma_1 C_{\delta}} \sigma_1 \\
&\leq  C_{\delta} \sigma_2 \left( \|f^1_\vecw\|_F + 2 \sqrt{C_2 h_1 \sigma_1 C_{\delta}} \sigma_1   \right) + C_3 \frac{\|D_2 W_2\|_F}{\sqrt{h_2}} 2 \sqrt{C_2 h_1 \sigma_1 C_{\delta}} \sigma_1   \\
& \stackrel{(i)}{\leq} C_{\delta} \sigma_2 \left(\|f^1_\vecw\|_F  + 2 \sqrt{\frac{\hat{\sigma_1}}{\sqrt{h_i + h_{i-1}}}} \hat{\sigma_1} \right) + 2 \hat{\sigma_1}  \frac{C_3  \|f_\vecw(x)\|_F^{\nicefrac{1}{d}}}{\mu} \sqrt{\frac{\hat{\sigma_1}}{\sqrt{h_i + h_{i-1}}}}\\
&=C_{\delta} \sigma_2 \left( \|f^1_\vecw\|_F + \beta_1 \hat{\sigma_1} \right) +  \frac{C_3  \|f_\vecw(x)\|_F^{\nicefrac{1}{d}}}{\mu} \beta_1 \hat{\sigma_1} 
\end{align*}
where, $\beta_i = 2 \sqrt{\frac{\hat{\sigma_1}}{\sqrt{h_i + h_{i-1}}}}$. $(i)$ follows from condition $C1$, which results in $\Pi_{i=2}^d \frac{\mu \|D_i W_i\|_F}{\sqrt{h_i}} \frac{\mu \|D_1 W_1 x\|_F}{\sqrt{h_1}} \leq \|f_\vecw(x)\|_F$. Hence, if we consider the rebalanced network\footnote{The parameters of ReLu networks can be scaled between layers without changing the function} where all layers have same values for $\frac{\mu \|D_i W_i\|_F}{\sqrt{h_i}} $, we get, $\frac{\mu \|D_i W_i\|_F}{\sqrt{h_i}} \leq \|f_\vecw(x)\|_F^{\nicefrac{1}{d}}$. Also the above equations follow from setting, $\sigma_i =\frac{\hat{\sigma}_i}{C_2 C_{\delta}\sqrt{h_i + h_{i-1}}}$. 

Hence, with probability $\geq 1-\frac{2\delta}{d}$,
\begin{align*}
\|\hatD_2 - D_2\|_1 &\leq C_2*h_2 \left( C_{\delta} \sigma_2 \left( \|f^1_\vecw\|_F + \beta_1 \hat{\sigma_1} \right) +  \frac{C_3  \|f_\vecw(x)\|_F^{\nicefrac{1}{d}}}{\mu}\beta_1 \hat{\sigma_1} \right).
\end{align*}

Since, we choose $\sigma_i$ to scale as some small number $O(\sigma)$, in the above expression the first term scales as $O(\sigma)$ and the last two terms decay at least as $O(\sigma^{\nicefrac{3}{2}})$. Hence we do not include them in the computation of $\err$.
 
\begin{align*}
\E \| \err_2 \|_F &= \E \|\hatD_2(W+\eps)_2 *\hatD_1*(W+\eps)_1 x - D_2(W+\eps)_2 *D_1*(W+\eps)_1 x \|_F \\
&\leq \E \| (\hatD_2-D_2)(W+\eps)_2 *(\hatD_1-D_1)*(W+\eps)_1 x\|_F  + \E \| D_2(W+\eps)_2 *(\hatD_1-D_1)*(W+\eps)_1 x\|_F \\ & \quad \quad+ \E \| (\hatD_2-D_2)(W+\eps)_2 *D_1*(W+\eps)_1 x\|_F.
\end{align*}

We will bound now the first term in the above expression. With probability $\geq 1-\frac{2\delta}{d}$,

\begin{align*}
\E & \| (\hatD_2-D_2)(W+\eps)_2 *(\hatD_1-D_1)*(W+\eps)_1 x\|_F \\
&\leq \E \| (\hatD_2-D_2)W_2 *(\hatD_1-D_1)*W_1 x\|_F + \E \| (\hatD_2-D_2)W_2 *(\hatD_1-D_1)*\eps_1 x\|_F \\ & \quad \quad + \E \| (\hatD_2-D_2)\eps_2 *(\hatD_1-D_1)*W_1 x\|_F + \E \| (\hatD_2-D_2)\eps_2 *(\hatD_1-D_1)*\eps_1 x\|_F \\
&\leq  2 \sqrt{C_2*h_2 C_{\delta} \sigma_2  \|f^1_W\|_F }C_{\delta} \sigma_2  \|f^1_W\|_F \sqrt{C_2*h_1 *C_{\delta} \sigma_1}C_{\delta} \sigma_1 \\ & \quad \quad + 2 \sqrt{C_2*h_2 C_{\delta} \sigma_2  \|f^1_W\|_F }C_{\delta} \sigma_2  \sqrt{h_1} \sqrt{C_2*h_1 *C_{\delta} \sigma_1}C_{\delta} \sigma_1 +O(\sigma^2)\\
&\leq  4 \|f_\vecw^2\|_F \frac{C^2_{\delta} \sigma_2 \sigma_1 \sqrt{h_1}}{\mu \|D_2 W_2\|_F} \Pi_{i=1}^2 \sqrt{C_2 h_i C_{\delta} \sigma_i}.
\end{align*}

%\begin{align*}
%\E \| \err_2 \|_F &= \E \|(W+\eps)_2 *\hatD_1*(W+\eps)_1 x - (W+\eps)_2 *D_1*(W+\eps)_1 x \|_F \\
%&= \E \| (W+\eps)_2 * \hatD_1 -D_1 *(W+\eps)_1 x\|_F  \\
%&\leq \E \| W_2 * \hatD_1 -D_1 *W_1 x\|_F + \E \| W_2 * \hatD_1 -D_1 *\eps_1 x\|_F \\ 
%& \quad \quad +\E \| \eps_2 * \hatD_1 -D_1 *W_1 x\|_F +\E \| \eps_2 * \hatD_1 -D_1 *\eps_1 x\|_F\\
%& \stackrel{(i)}{\leq} \|W_2\|_F \sqrt{C_2 h_1 \sigma_1 C_{\delta}} \sigma_1 C_{\delta}  + \|W_2\|_F \sqrt{C_2 h_1 \sigma_1 C_{\delta}} \sigma_1 \\ & \quad \quad + \sigma_2 \sqrt{h_2}\sqrt{C_2 h_1 \sigma_1 C_{\delta}} \sigma_1 C_{\delta} + \sigma_2 \sqrt{h_2}\sqrt{C_2 h_1 \sigma_1 C_{\delta}} \sigma_1 \\
%& \stackrel{(ii)}{\leq } 
%\end{align*}
%$(ii)$ follows from $\|W_2\|_F \|D_1 W_1 x\|_F \leq \sqrt{h_1} f^2_W(x)$ (condition C1).

\noindent {\bf Induction step:}

Now we assume the statement for all $i \leq k$ and prove it for $k+1$.
$\|\hatD_{k} -D_{k}\|_1 \leq O \left( C_2 h_k  C_{\delta} \sigma_k  \|f^{k-1}_\vecw\|_F \right) $ and $\E\| \err_k\|_F \leq   O \left( \Pi_{i=1}^k \left( 1+\frac{\sigma_i\sqrt{h_i}\sqrt{h_{i-1}} C_2 C_3 }{\mu^2 \|W_i\|_F  }\right)\left(\Pi_{i=1}^k (1+ \frac{\sigma_i\sqrt{h_i}\sqrt{h_{i-1}} C_{\delta} C_2}{\mu^2 \|W_i\|_F } ) - 1\right)\|f^k_\vecw\|_F \right)$. Now we prove the statement for $k+1$.
 
\begin{align*}
\|\hatD_{k+1} - D_{k+1}\|_1 &= \sum_i 1_{\inner{(W+\eps)_{k+1 , i}}{ \Pi_{i=1}^{k } \hatD_i (W+\eps)_i *x } *\inner{W_{2,i}}{D_1 W_1 x } \leq 0} \\
&\leq \sum_i 1_{ \abs{ \inner{W_{k+1 , i}}{ \Pi_{i=1}^{k } \hatD_i (W+\eps)_i *x }} \leq \abs{ \inner{\eps_{k+1 , i}}{ \Pi_{i=1}^{k } \hatD_i (W+\eps)_i *x }} } \\
&=  \sum_i 1_{ \abs{ \inner{W_{k+1 , i}}{ f^k_{\vecw+\eps} }} \leq \abs{ \inner{\eps_{k+1 , i}}{ f^k_{\vecw+\eps} }}} \\
&\leq   \sum_i 1_{ \abs{ \inner{W_{k+1 , i}}{ f^k_{W} }} \leq \abs{ \inner{\eps_{k+1 , i}}{ f^k_{\vecw} }} + \abs{ \inner{\eps_{k+1 , i}}{ f^k_{\vecw +\eps} -f^k_\vecw }} + \abs{ \inner{W_{k+1 , i}}{ f^k_{\vecw+\eps} -f^k_{\vecw} }} } 
\end{align*}

Hence, with  probability $\geq 1-\frac{k\delta}{d}$, 
\begin{align*}
\|\hatD_{k+1} - D_{k+1}\|_1 &\leq  C_2 h_{k+1} \left[ C_{\delta} \sigma_{k+1} ( \|f^k_{\vecw} \|_F +\|f^k_{\vecw+\eps} -f^k_\vecw \|_F ) + \| W_{k+1, i}\|  \|f^k_{\vecw+\eps} -f^k_\vecw \|_F \right] \\
&\leq C_2 h_{k+1}C_{\delta} \sigma_{k+1} \|f^k_{\vecw} \|_F +C_2 h_{k+1}C_{\delta} \sigma_{k+1}\|f^k_{\vecw+\eps} -f^k_\vecw \|_F  +C_2 h_{k+1} \| W_{k+1, i}\|  \|f^k_{\vecw+\eps} -f^k_\vecw \|_F . 
\end{align*}
Now we will show that the last two terms in the above expression scale as $O(\sigma^2)$. For that, first notice that $\|f^k_{\vecw+\eps} -f^k_\vecw \|_F \leq  \left(\Pi_{i=1}^k \left(1+ \frac{\sigma_i \sqrt{h_i h_{i-1}}}{\mu^2 \|D_i W_i\|_F} \right) -1 \right)  \|f_\vecw(x)\|_F + \err_k$, from lemma~\ref{lem:linear}. Note that the second term in the above expression clearly scale as $O(\sigma^2)$.

Hence,
\begin{align*}
\|\hatD_{k+1} - D_{k+1}\|_1 \leq O \left( C_2 h_{k+1}C_{\delta} \sigma_{k+1} \|f^k_{\vecw} \|_F \right).
\end{align*}

%\begin{align*}
%+C_2 h_{k+1}C_{\delta} \sigma_{k+1}\|f^k_{W+\eps} -f^k_W \|_F  +C_2 h_{k+1} \| W_{k+1, i}\|  \|f^k_{W+\eps} -f^k_W \|_F &\leq 
%\end{align*}

\begin{align*}
\|\err_{k+1} \| &= \|f^{k+1}_{\vecw+\eps} -\tilde{f}^{k+1}_{\vecw+\eps}\|_F \\ &= \| \hatD_{k+1}(W+\eps)_{k+1} \Pi_{i=1}^{k+1} \hatD_i (W+\eps)_i x - D_{k+1}(W+\eps)_{k+1} \Pi_{i=1}^{k+1} D_i (W+\eps)_i x\|_F \\
&\leq \| (\hatD_{k+1}-D_{k+1})(W+\eps)_{k+1} \Pi_{i=1}^{k+1} D_i (W+\eps)_i x\|_F + \| \hatD_{k+1} (W+\eps)_{k+1} \err_k\|_F \\
&\leq \| (\hatD_{k+1}-D_{k+1})(W+\eps)_{k+1} \Pi_{i=1}^{k+1} D_i (W+\eps)_i x\|_F + \| (\hatD_{k+1}-D_{k+1}) (W+\eps)_{k+1} \err_k\|_F \\ & \quad \quad +\|D_{k+1} (W+\eps)_{k+1} \err_k\|_F 
\end{align*}

Substituting the bounds for $\hatD_{k+1} -D_{k+1}$ and $\err_k$ gives us, with probability $\geq 1-\frac{k \delta}{d}$.
\begin{align*}
\E \|\err_{k+1} \| &\leq  \sqrt{C_2 h_{k+1}C_{\delta} \sigma_{k+1} \|f^k_{\vecw} \|_F } C_{\delta} \sigma_{k+1} \|f^k_{W} \|_F \E \| \Pi_{i=1}^{k+1} D_i (W+\eps)_i x\|_F \\
&\quad \quad + \E \|\err_k\|_F \left( \sqrt{C_2 h_{k+1}C_{\delta} \sigma_{k+1} \|f^k_{\vecw} \|_F } C_{\delta} \sigma_{k+1} \|f^k_{\vecw} \|_F + \|D_{k+1} W_{k+1}\|_F +\sigma_{k+1}\sqrt{h_{k+1}} \right)
\end{align*}
Now we bound the above terms following the same approach as in proof of Lemma~\ref{lem:linear}, by considering all possible replacements of $W_i$ with $\eps_i$. That gives us the result.

%&\leq \|\sum_{s \in [d], |s| \neq 0}  f^{k+1}_{W+\eps}(\hatD_{-i},  \hatD_i-D_i , i\in s)\|_F.
%$f^{k+1}_{W+\eps}(\hatD_{-i},  \hatD_i-D_i )$ is $f^{k+1}_{W+\eps}$ with $\hatD_i$ replaced by $\hatD_i-D_i$.

%Now we will bound $\|f^{k+1}_{W+\eps}(\hatD_{-i},  \hatD_i-D_i , i\in s)\|_F$, for a given $s$ following the same approach as in proof of Lemma~\ref{lem:linear}, by considering all possible replacements of $W_i$ with $\eps_i$.
%\begin{align*}
%\| f^{k+1}_{W+\eps}(\hatD_{-i},  \hatD_i-D_i , i\in s)\|_F \leq \|f^{k+1}_W(x)\|_F
%\end{align*}
%
%Hence, 
%\begin{align*}
%\|\err_{k} \| &\leq \sum_{s \in [d], |s| \neq 0}\left(C_2 \sqrt{2\sigma f_\vecw}\right)^{|s|} \left( \sum_{i=1}^d \binom{d}{i}(1+\ln(1/\delta))^i \frac{\sigma^i}{c^i} \right)\|f^k_W(x)\| \\
%& \leq C (d\sigma +O(d\sigma)) \|f^k_{W}\|_F 
%\end{align*}

\end{proof}

\section{Supporting results}

\begin{lem}\label{lem:gauss_product}
Let $A$ ,$B$ be $n_1 \times n_2$ and $n_3 \times n_4$ matrices and $\eps$ be a $n_2\times n_3$ entrywise random Gaussian matrix with $\eps_{ij} \sim \mathcal{N}(0,\sigma)$. Then,
$$\E \left[\| A*\eps*B\|_F \right] \leq \sigma \|A\|_F  \|B\|_F .$$
\end{lem}
\begin{proof}
By Jensen's inequality, 
\begin{align*}
\E \left[\| A*\eps*B\|_F \right]^2 &\leq \E \left[\| A*\eps*B\|_F^2 \right] \\
&= \E \left[\left(\sum_{ij} \sum_{kl} A_{ik} \eps_{kl} B_{lj} \right)^2 \right]\\
&= \sum_{ij} \sum_{kl} A_{ik}^2 \E \left[\eps_{kl}^2\right] B_{lj}^2 \\
&= \sigma^2 \|A\|_F^2 \| B\|_F^2.
\end{align*}

\end{proof}

%&= \| (W+\eps)_{k+1} \hatD_k \left( (W+\eps)_k \Pi_{i=1}^{k-1}  \hatD_i (W+\eps)_i x - (W+\eps)_k \Pi_{i=1}^k D_i (W+\eps)_i  x \right)\|_F \\ & \quad \quad + \| (W+\eps)_{k+1} * (\hatD_k -D_k) * (W+\eps)_k \Pi_{i=1}^k D_i (W+\eps)_i  x\|_F \\
%&\leq \| (W+\eps)_{k+1} \left( f^{k}_{W+\eps} -\tilde{f}^{k}_{W+\eps} \right) \|_F + \| (W+\eps)_{k+1}\|_F \|\hatD_{k} - D_k\| \| \tilde{f}^k\| \\
%&\leq \| (W+\eps)_{k+1}\|_2 \|\err_k \| _F + \| (W+\eps)_{k+1}\|_{\infty} \|\hatD_{k} - D_k\|_1  \| \tilde{f}^k\|_{\infty} 

\section{Conditions in Theorem ~\ref{thm:relu}}
In this section, we compare the conditions in Theorem ~\ref{thm:relu} of a learned network with that of its random initialization. We trained a 10-layer feedforward network with 1000 hidden units
in each layer on MNIST dataset. Figures~\ref{fig:conditions_verify2}, \ref{fig:conditions_verify3} and \ref{fig:conditions_verify4} compare condition $C1$, $C2$ and $C3$ on learned weights to that of random initialization respectively. Interestingly,
we observe that the network with learned weights is very similar to its random initialization in terms of these conditions.
\begin{figure}[t]
\centering
\includegraphics[width=.32\textwidth]{\figdir/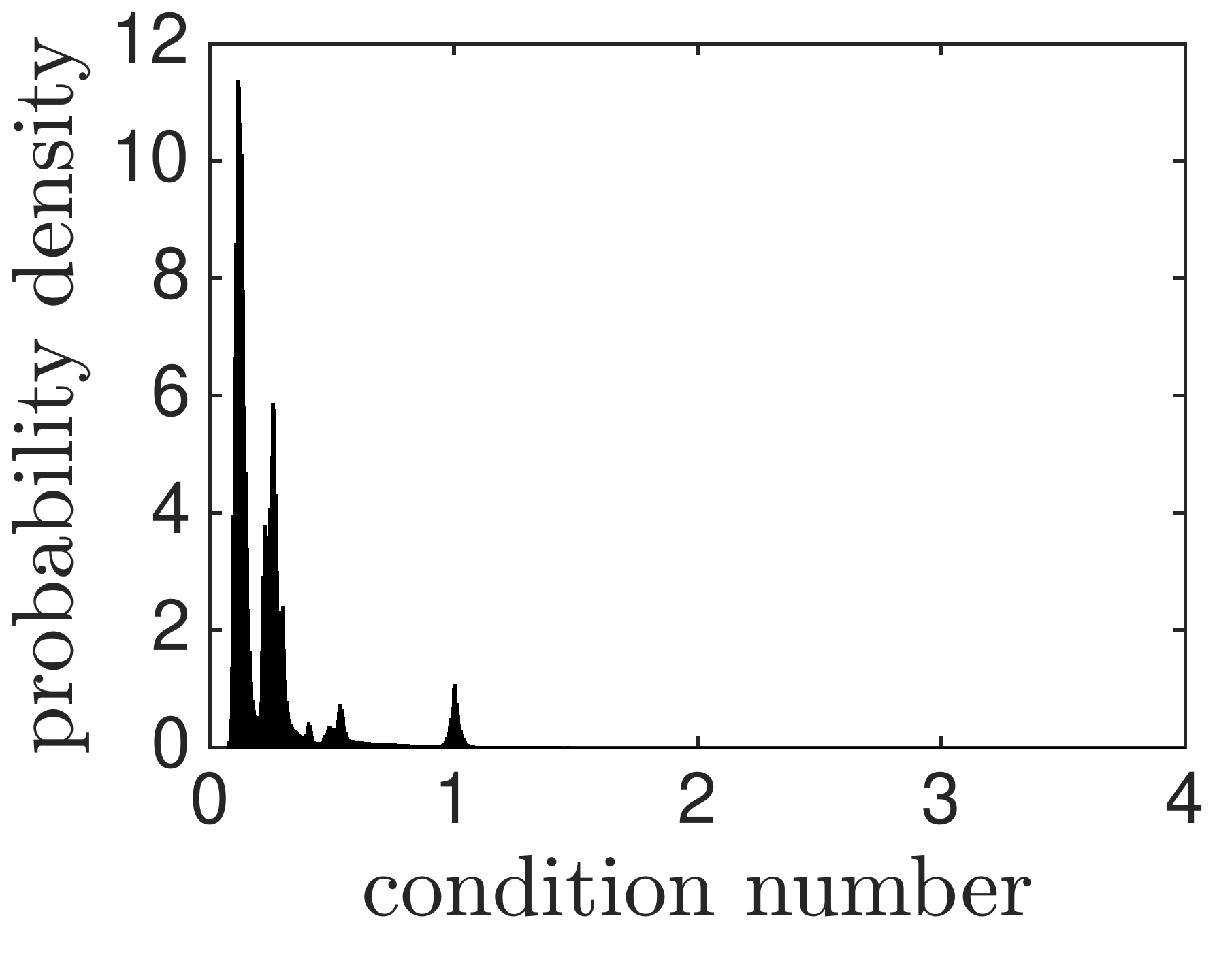}
\includegraphics[width=.32\textwidth]{\figdir/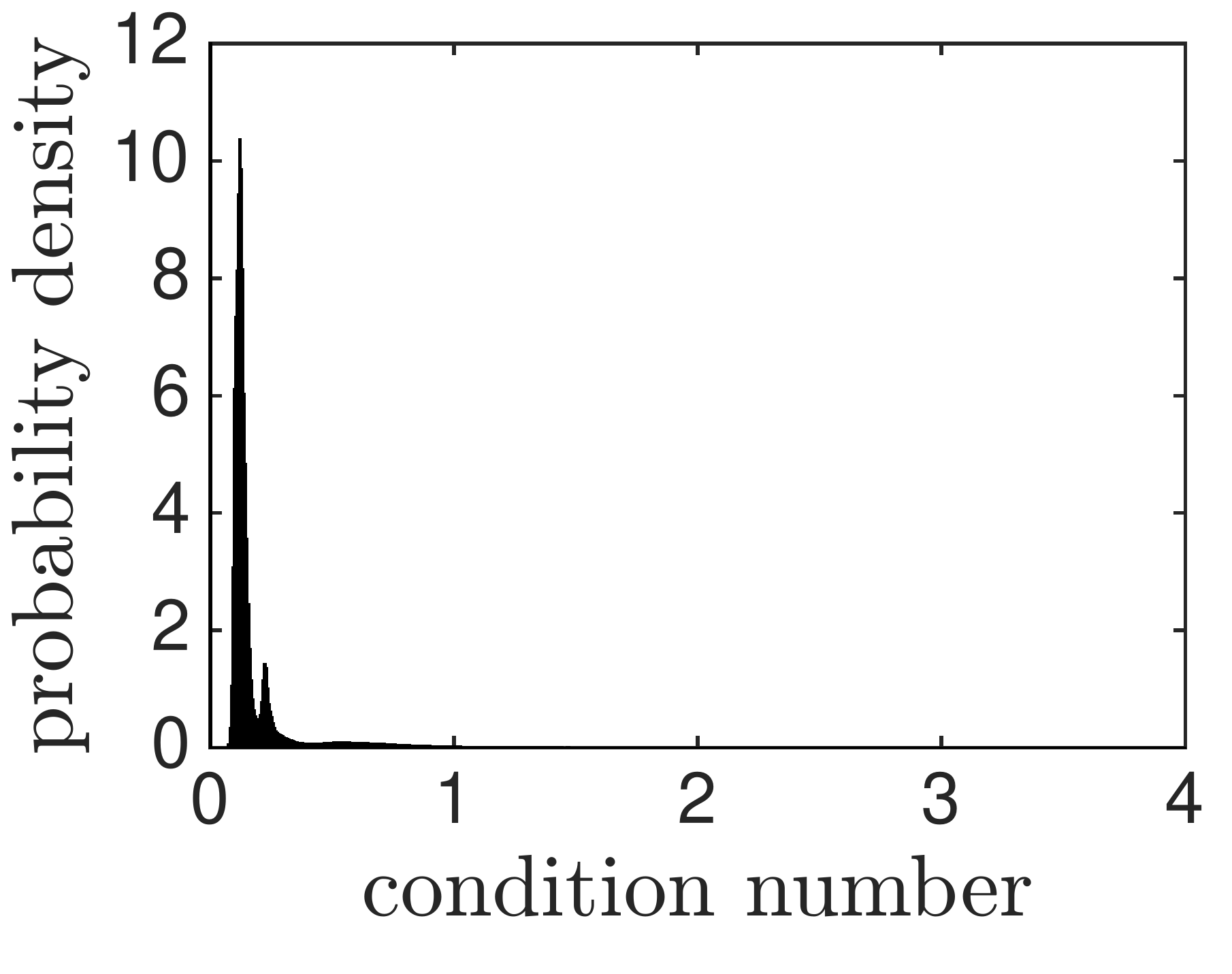}
\includegraphics[width=.32\textwidth]{\figdir/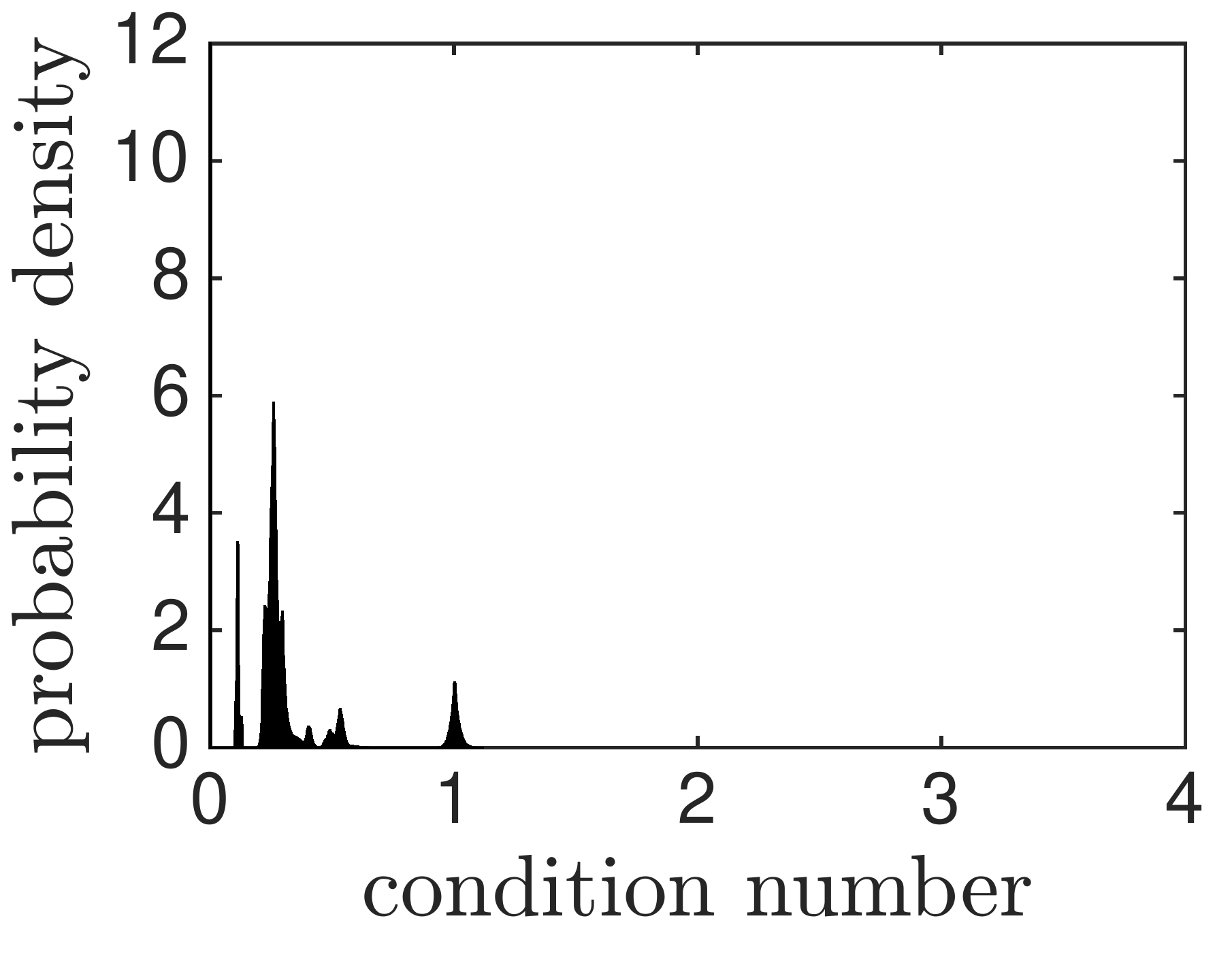}\\
\includegraphics[width=.32\textwidth]{\figdir/hist.pdf}
\includegraphics[width=.32\textwidth]{\figdir/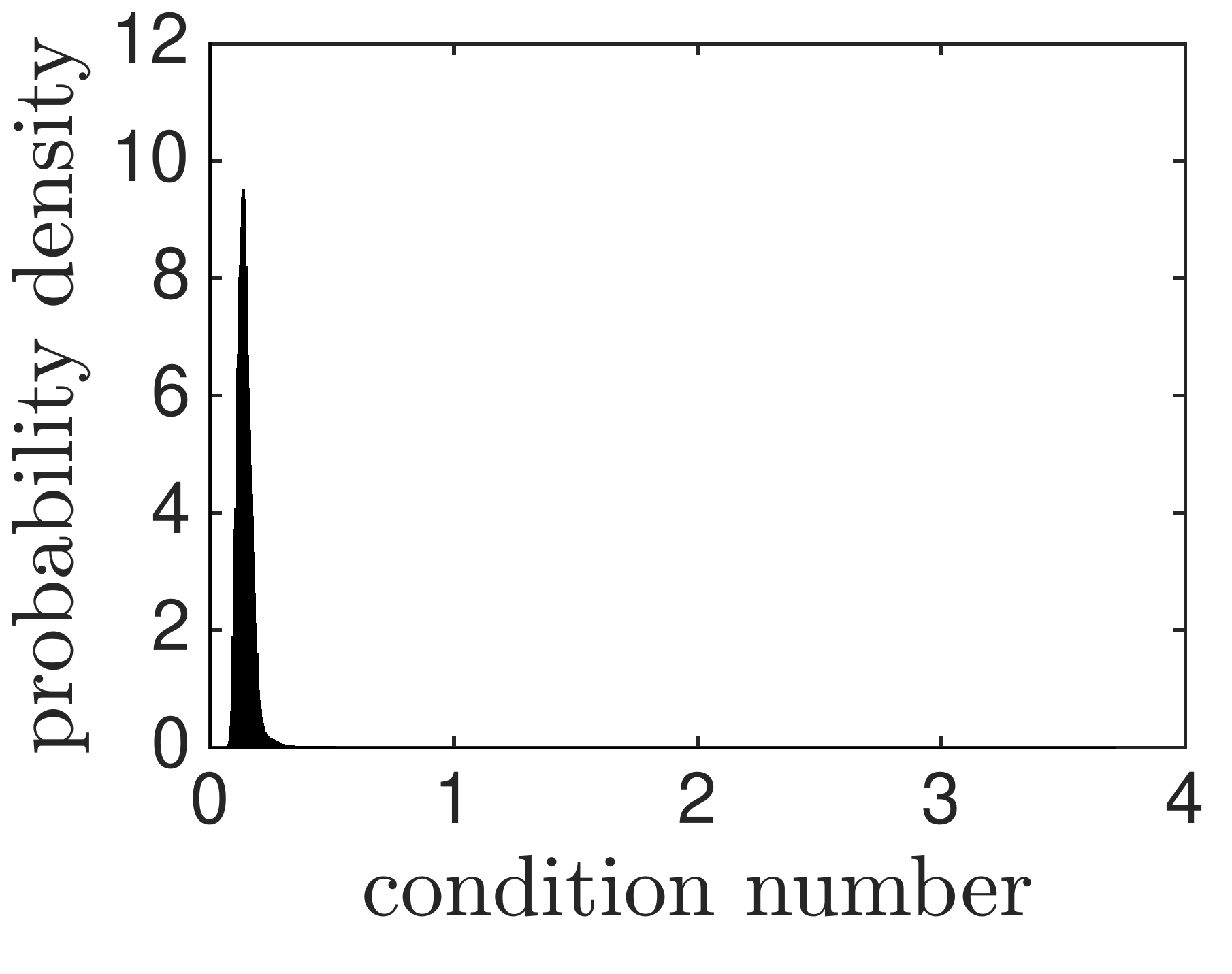}
\includegraphics[width=.32\textwidth]{\figdir/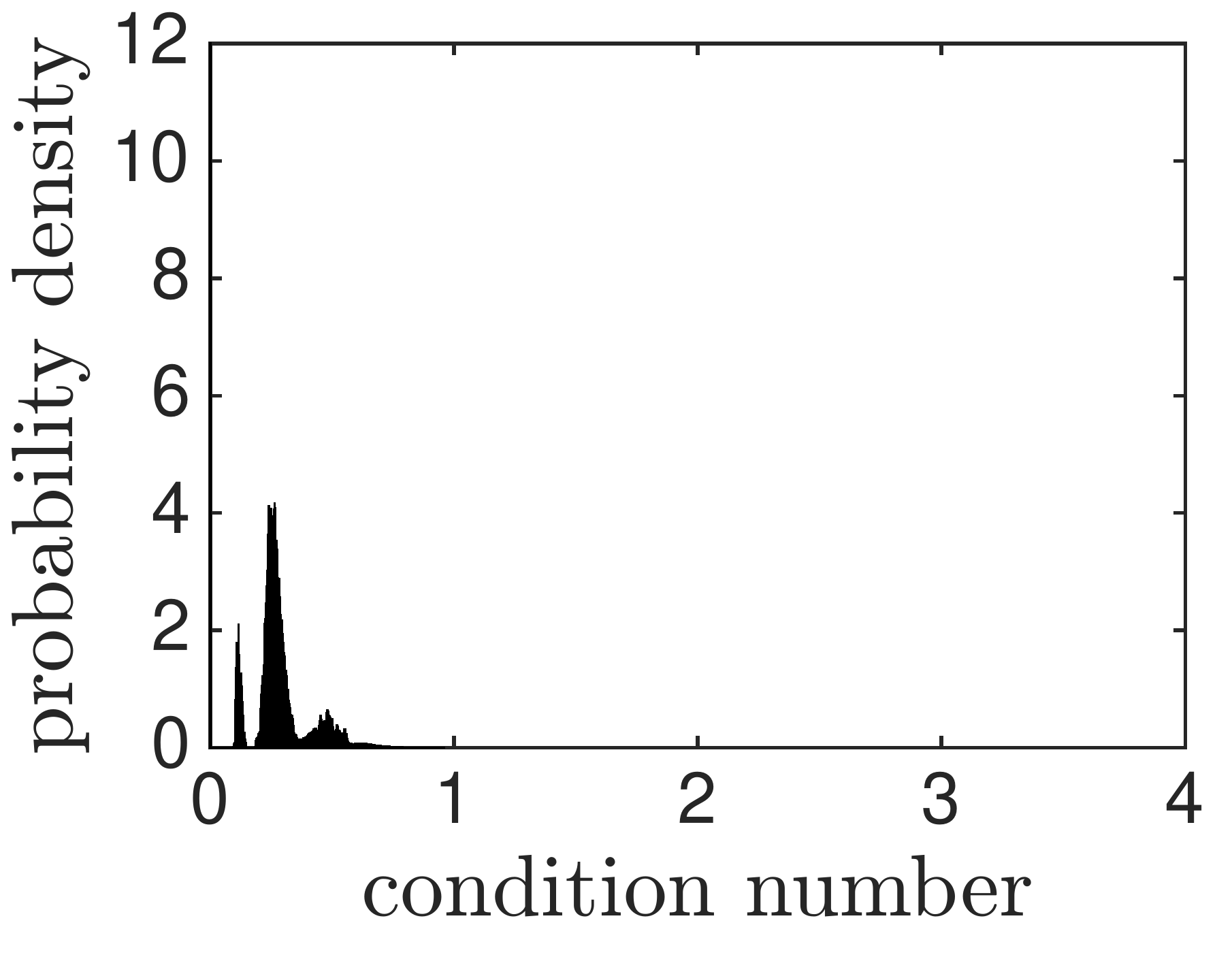}\\
\begin{picture}(0,0)(0,0)
{\small \put(-150, 215){$a\leq c\leq b-1$}\put(-20, 215){$a< c< b-1$}\put(100, 215){$c=a\;$ or $\;c=b-1$}}
\end{picture}
\caption{\small Condition $C1$: condition number $\frac{1}{\mu}$ of the network and its decomposition to two cases for random initialization and learned weights. \textbf{Top}: random initialization \textbf{Bottom}: learned weights. \textbf{Left}: distribution of all combinations of $a\leq c\leq b-1$. \textbf{Middle}: when $a<c<b-1$. \textbf{Right}: when $c=a$ or $c=b-1$.}
\label{fig:conditions_verify2}
\end{figure}

\begin{figure}[t]
\centering
\includegraphics[width=.32\textwidth]{\figdir/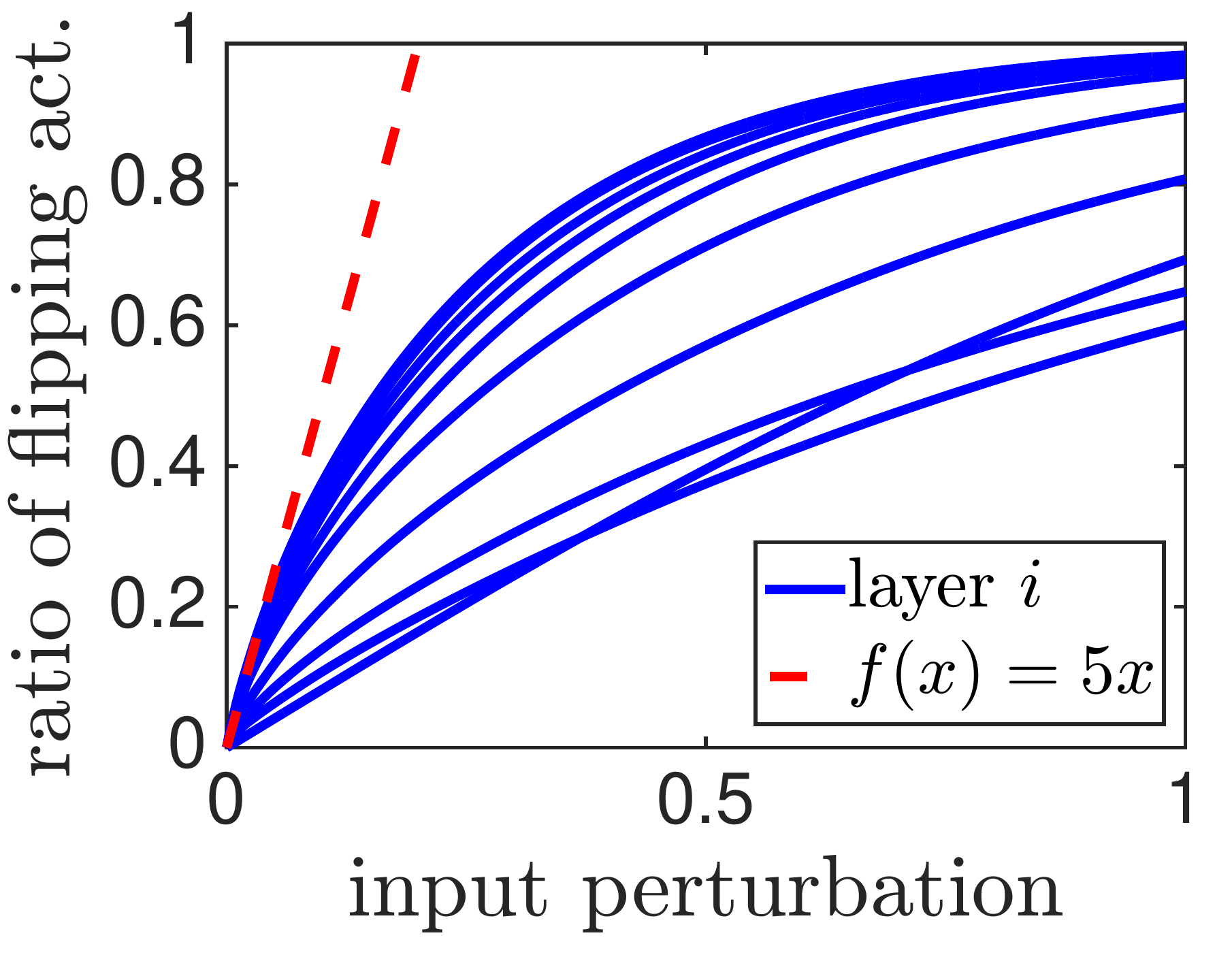}
\includegraphics[width=.32\textwidth]{\figdir/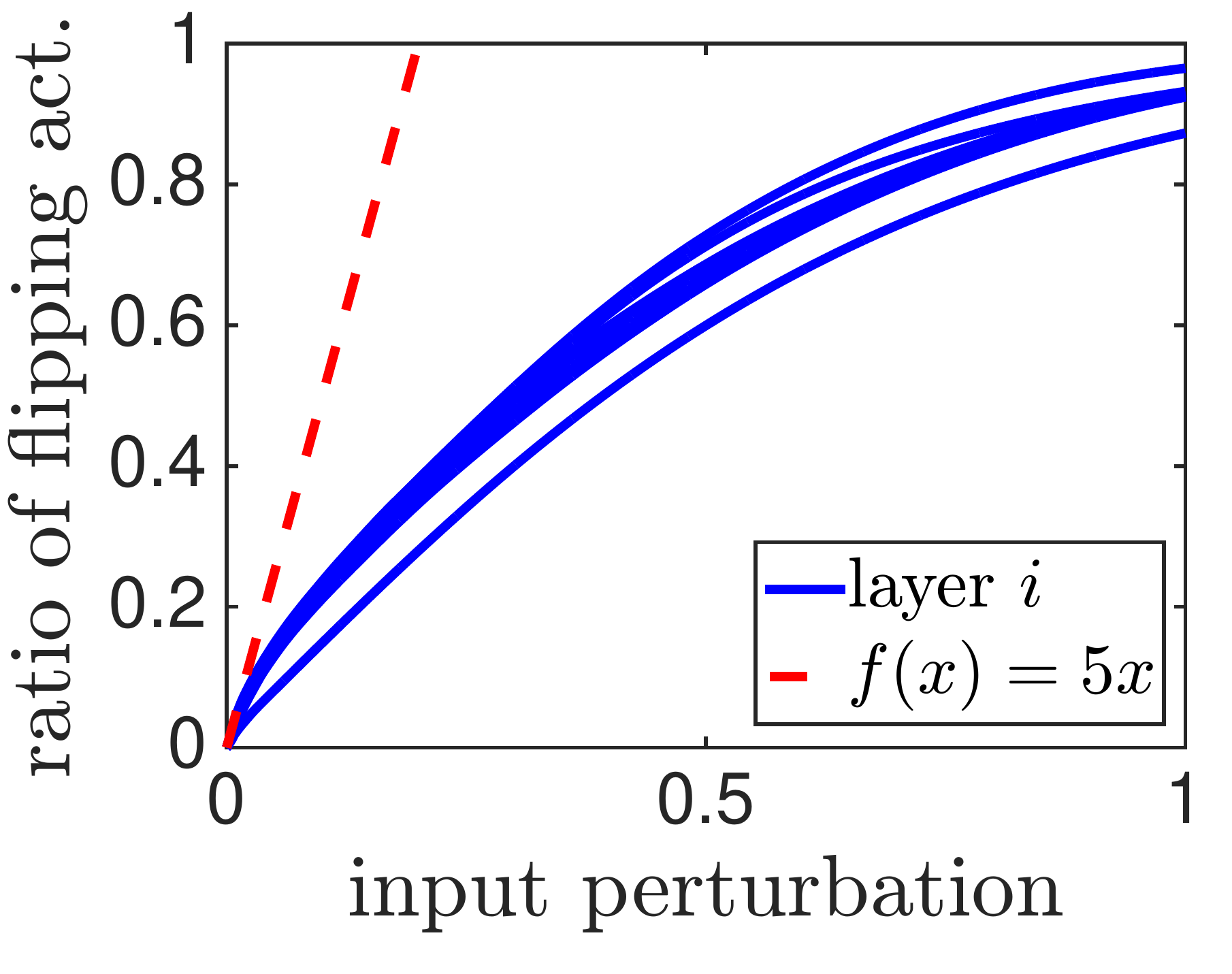}
\includegraphics[width=.33\textwidth]{\figdir/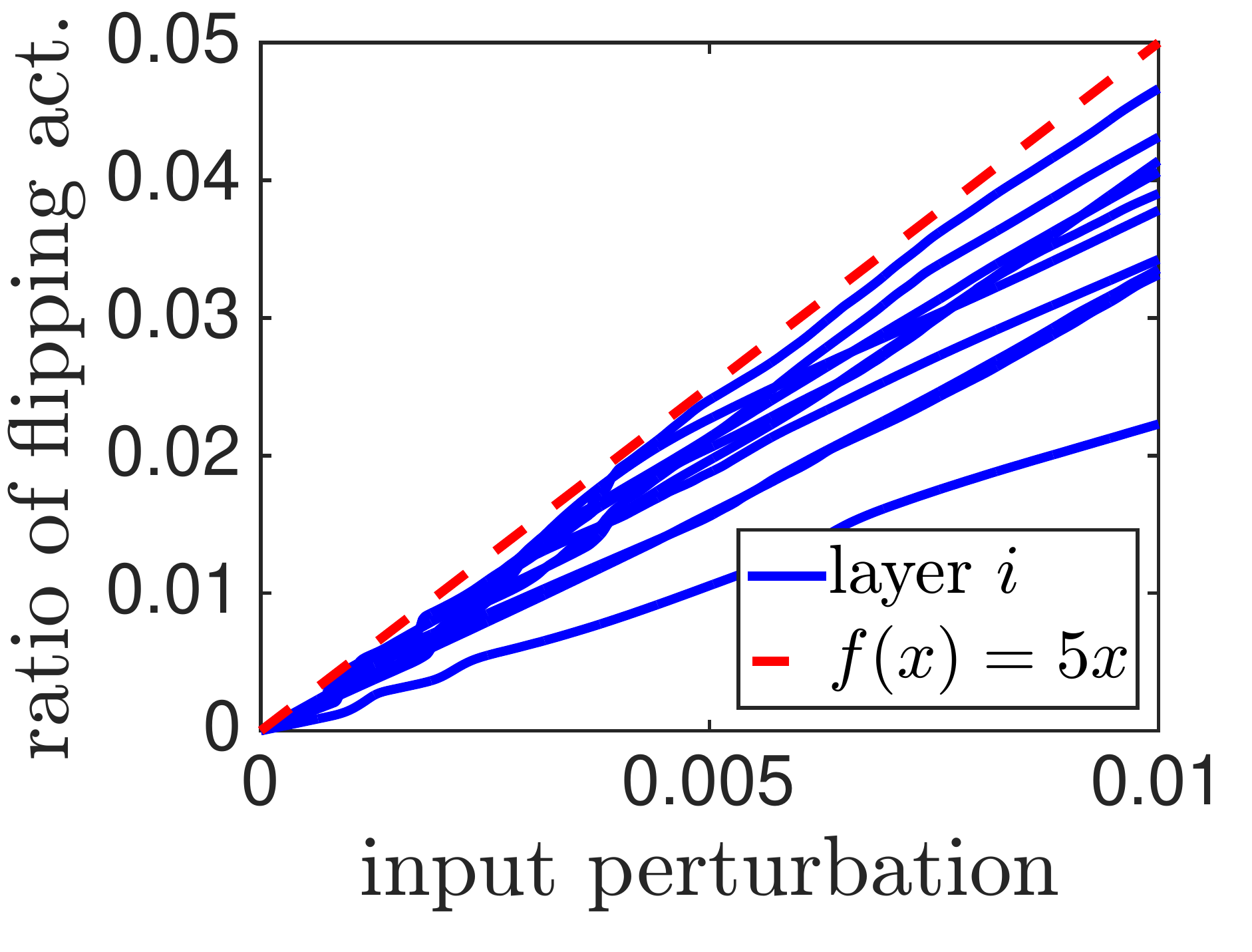}
\caption{\small Ratio of activations that flip based on the magnitude of perturbation. \textbf{Left}: random initialization. \textbf{Middle}: learned weights. \textbf{Right}: learned weights (zoomed in).}
\label{fig:conditions_verify3}
\end{figure}

\begin{figure}[t]
\centering
\includegraphics[width=.24\textwidth]{\figdir/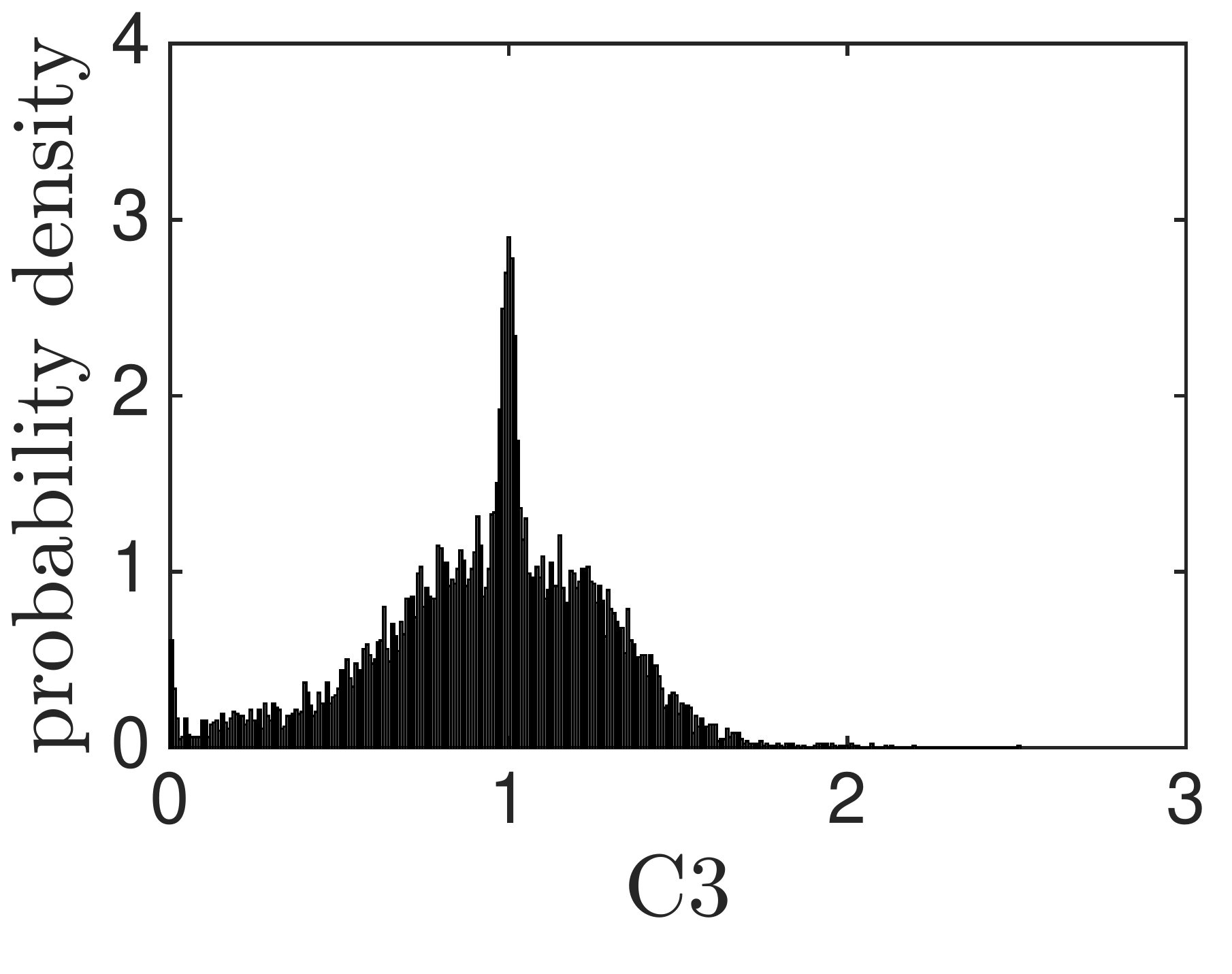}
\includegraphics[width=.24\textwidth]{\figdir/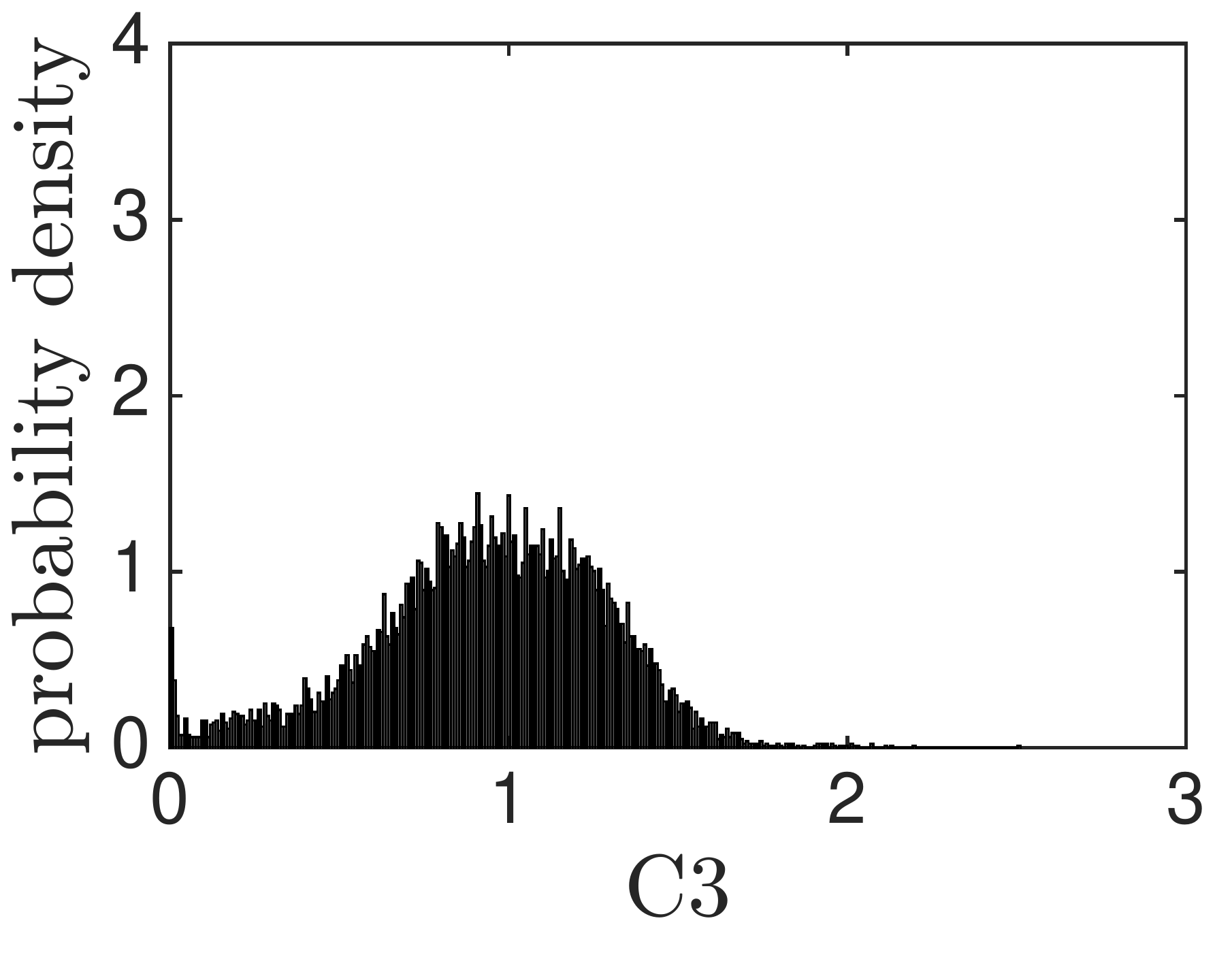}
\includegraphics[width=.25\textwidth]{\figdir/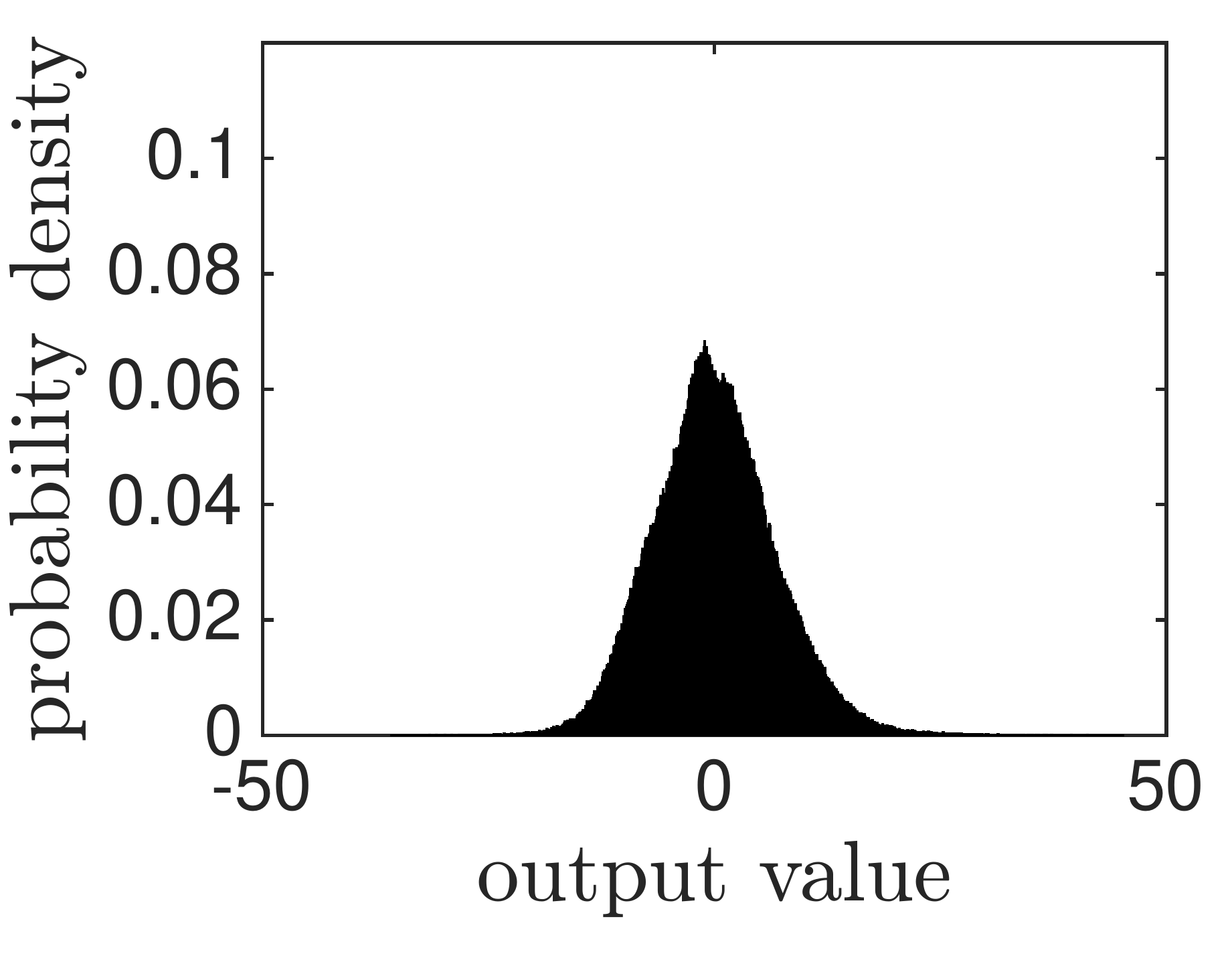}
\includegraphics[width=.25\textwidth]{\figdir/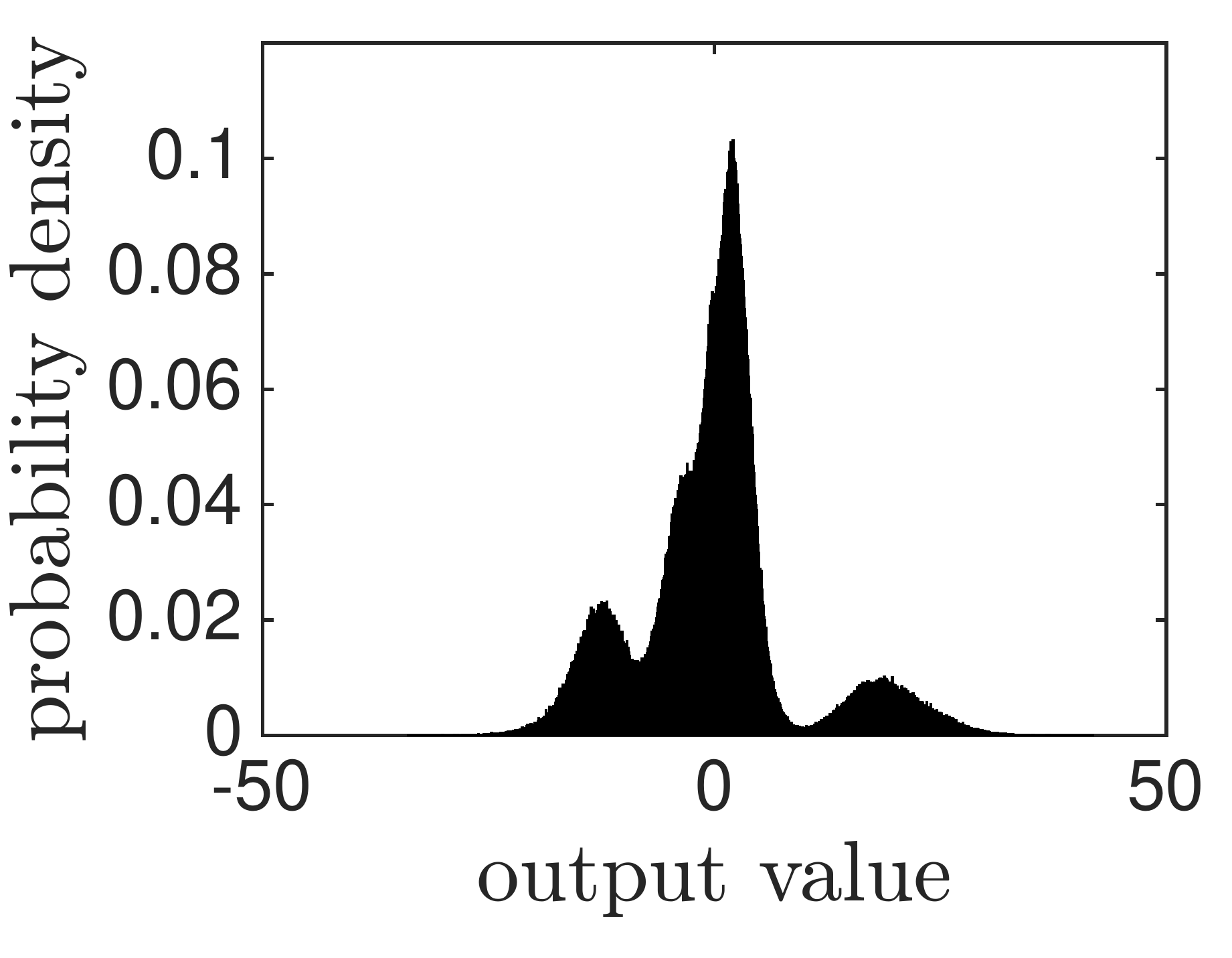}
\caption{\small From left to right: Condition $C3$ for random initialization and learned network, output values for random and learned network}
\label{fig:conditions_verify4}
\end{figure}

%\section{More experiments with random labels}
%
%\begin{figure}[t]
%\centering
%\includegraphics[width=.45\textwidth]{\figdir/correct-cifar10-linear-lin_log-overal-L2-norm_trsize.pdf}
%\includegraphics[width=.46\textwidth]{\figdir/cifar10-linear-lin_log-overal-L2-norm_trsize.pdf}
%\caption{\small A convolutional network trained on subsets of CIFAR-10 dataset of different sizes with either true or random labels. The left panel corresponds to the training with true labels and the right panel corresponds to training with random labels. }
%\label{fig:linear-cifar}
%\end{figure}

\end{document}